\documentclass{article}

\usepackage{xcolor}
\usepackage{amsmath}
\usepackage{bm}
\usepackage{enumitem}
\usepackage{amsthm}

\DeclareMathOperator*{\argmin}{arg\,min}

\newtheorem{assumption}{Assumption}

\newtheorem{theorem}{Theorem}[section]
\newtheorem{proposition}[theorem]{Proposition}

\newtheorem{corollary}[theorem]{Corollary}
\newtheorem{remark}[theorem]{Remark}

\usepackage{PRIMEarxiv}

\usepackage[utf8]{inputenc} 
\usepackage[T1]{fontenc}    
\usepackage{hyperref}       
\usepackage{url}            
\usepackage{booktabs}       
\usepackage{amsfonts}       
\usepackage{nicefrac}       
\usepackage{microtype}      
\usepackage{lipsum}
\usepackage{fancyhdr}       
\usepackage{graphicx}       
\graphicspath{{media/}}     

\title{Optimal Implicit Bias in Linear Regression}

\author{
  K Nithin Varma, Babak Hassibi \\
  California Institute of Technology\\
  Pasadena\\
  \texttt{\{nkanumur, hassibi\}@caltech.edu} \\
}

\begin{document}
\maketitle

\begin{abstract}
Most modern learning problems are over-parameterized, where the number of learnable parameters is much greater than the number of training data points. In this over-parameterized regime, the training loss typically has infinitely many global optima that completely interpolate the data with varying generalization performance. The particular global optimum we converge to depends on the implicit bias of the optimization algorithm. The question we address in this paper is, ``What is the implicit bias that leads to the best generalization performance?". To find the optimal implicit bias, we provide a precise asymptotic analysis of the generalization performance of interpolators obtained from the minimization of convex functions/potentials for over-parameterized linear regression with non-isotropic Gaussian data. In particular, we obtain a tight lower bound on the best generalization error possible among this class of interpolators in terms of the over-parameterization ratio, the variance of the noise in the labels, the eigenspectrum of the data covariance, and the underlying distribution of the parameter to be estimated. Finally, we find the optimal convex implicit bias that achieves this lower bound under certain sufficient conditions involving the log-concavity of the distribution of a Gaussian convolved with the prior of the true underlying parameter.
\end{abstract}

\section{Introduction}

Classical statistical learning theory mainly studies problems in data-rich regimes where the number of data points is much greater than the number of unknown parameters to be learned/estimated. In contrast, most modern learning problems like deep learning are typically highly overparameterized, i.e., the problem dimension $n$ to be much greater than the number of training data points $m$. Due to this over-parameterization, these models possess the capacity to completely fit any set of training data (even possibly random) \cite{zhang2016understanding}. Despite this overfitting, these models surprisingly generalize well on unseen data, and this so-called \emph{double descent} phenomenon was observed, for example, in \cite{belkin2019reconciling,muthukumar2020harmless}.  In this so-called \emph{interpolating} regime \cite{ma2018power} of over-parameterized models, there generally exist (infinitely) many global optima of weights that interpolate the data with varying generalization properties. The particular convergent global optima depend on the choice of the optimization algorithm used, which, therefore, determines its generalization performance. This inherent tendency of the optimizer towards certain global optima is termed its \emph{implicit bias}, and the success of deep learning has been mainly attributed to the implicit bias properties of the commonly used optimization algorithms, like stochastic gradient descent (SGD) and its variants \cite{neyshabur2017geometry}.

The implicit bias of gradient descent (GD) and SGD was studied in \cite{soudry2018implicit,gunasekar2018characterizing,gunasekar2018implicit}. Mirror Descent (MD) \cite{nemirovski1983problem}, and its stochastic counterpart, stochastic mirror descent (SMD), generalize (S)GD to a family of gradient-based algorithms, where the gradient updates happen in the so-called \emph{mirrored} domain characterized by a differentiable convex potential function $\Psi$. In the special case, when the potential is the squared Euclidean norm, SMD recovers SGD. Similar to SGD, SMD enjoys implicit regularization properties, and it was shown in \cite{gunasekar2018characterizing,azizan2019stochastic} that SMD converges to the interpolating global optima that minimize its convex potential $\Psi$ for linear models under regression losses with appropriate conditions on initialization. To be precise, if we consider the data $\{\bm x_i,y_i\}_{i=1}^{m}$, where $\bm x_i \in \mathbb{R}^n$ and $y_i \in \mathbb{R}$ for $i\in [m]$ and a strictly convex differentiable potential $\Psi: \mathbb{R}^n \xrightarrow[]{} \mathbb{R}$, when $n>m$, it was shown that when SMD is initialized at $\bm{\beta}_0$ such that $\bm{\beta}_0 := \argmin_{\bm{\beta}} \Psi(\bm{\beta})$ or equivalently $\nabla \Psi(\bm{\beta}_0) = 0$, it converges to the solution $\bm{\hat \beta} \in \mathbb{R}^d$ defined as 
\begin{equation}
\label{eq: implicit bias}
 \bm{\hat \beta}:= \argmin_{\bm \beta \in \mathbb{R}^n} \Psi(\bm \beta) \quad \text{s.t} \quad y_i = \bm x_i^{T} \bm\beta \;\text{for $i \in [m]$}.  
\end{equation} 
When $\Psi$ is chosen as the Euclidean norm squared, we see that SGD converges to the minimum $\ell_2$ normed interpolating solution. These results were also shown to approximately translate to non-linear models (like deep learning), where SMD converges close to the minimum potential interpolating solution \cite{azizan2021stochastic}. In \cite{azizan2019stochastic}, it was empirically observed that the generalization ability of the ResNet-18 network (with 11 million weights) trained on the CIFAR-10 dataset varied across different potentials of SMD as shown in Figure \ref{fig: ge cifar}. It was seen that $\ell_{1}$-norm SMD, which gave a sparse solution, performed relatively worse compared to SGD, and surprisingly, the $\ell_{10}$-norm SMD performed the best and even outperformed SGD. Extensive numerical evaluations of SMD were also done in \cite{sun2022mirror,sun2023unified} across different models and datasets, and it was observed that different potentials of SMD had varying relative performance across different models and datasets. 
\begin{figure}[h]
    \label{fig: ge cifar}
    \centering
    \includegraphics[width=.5\textwidth]{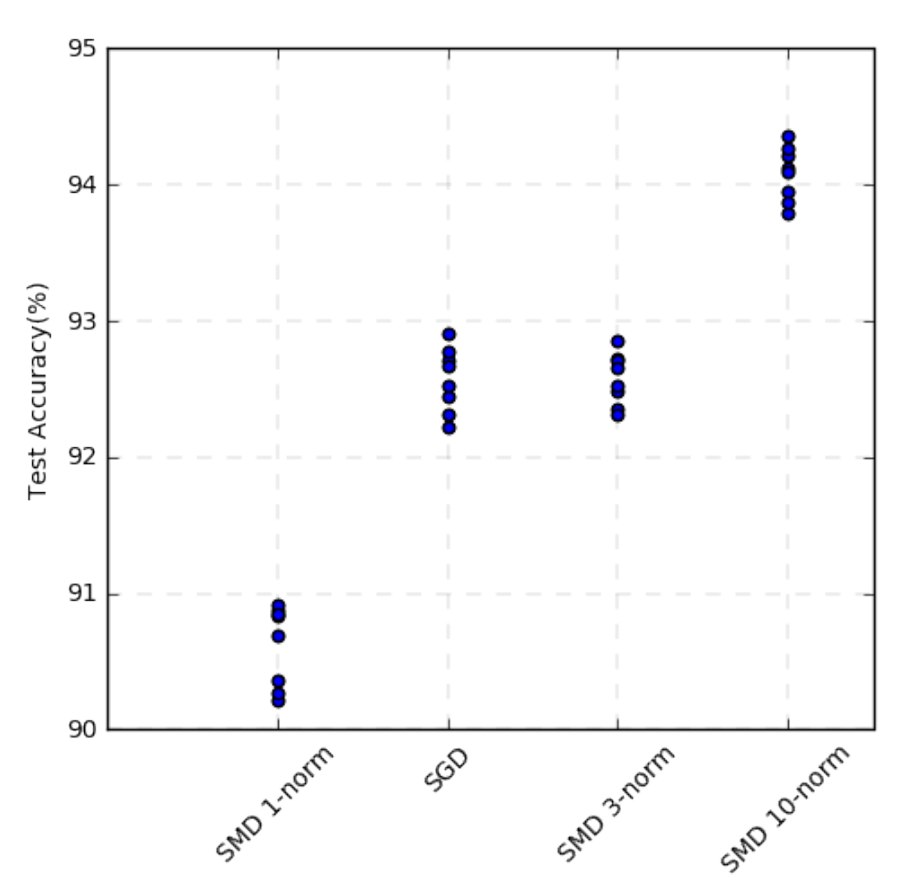}
    \caption{Generalization performance of ResNet-18 on the CIFAR-10 using different SMD algorithms. \cite{azizan2021stochastic}.}
\end{figure}
This dependence of the generalization performance on the choice of convex potential, the structure of the data, and the architecture of the model used remains an open problem, and little is understood, even theoretically, about the generalization performance for linear models for general convex potentials as implicit bias.

In this paper, we take a first step towards understanding the surprising phenomenon observed in Figure \ref{fig: ge cifar} by studying the problem of linear regression in the overparameterized regime. We characterize the precise asymptotic generalization performance of the solution in \eqref{eq: implicit bias} and provide answers to the following question:
$$
\text{\emph{What is the optimal choice of implicit bias that achieves the minimal generalization error?} }
$$
Previously, \cite{amari2020does} studied the role of preconditioner in preconditioned gradient descent (PGD) on the generalization error of interpolators obtained from linear regression. Later, \cite{oravkin2021optimal} made their arguments precise and characterized the optimal interpolator among the class of interpolators that can be obtained as a linear function labels/outputs $y$, which ends up being the implicit bias of PGD with an appropriately chosen preconditioner and initialization point. Since PGD is a special case of MD with quadratic convex potentials, their results on the optimal interpolator depend only on the second moments of the underlying unknown signal. In contrast, our work considers fundamental limits on a much broader class of interpolators that are obtained as minimizers of convex functions subject to interpolation. Due to this richness in the class of interpolators, as we'll see, our analysis and results depend more closely on the distribution of the underlying signal through quantities like Fisher information rather than just the second moments.

\subsection{Notation}
We use $[k]$ to denote the set $\{1,\cdots,k\}$. We use boldface lowercase letters $\bm{x},\bm{y},\bm{\mu},\cdots$ to denote vectors and boldface uppercase letters $\bm{X},\bm{Y},\bm{M},\cdots$ for matrices. $\bm{I_n}\in \mathbb{R}^{n\times n}$ denotes the identity matrix. For a random variable $B$ with a differentiable density $P_{B}(b)$, we define its \emph{Fisher Information} $\mathcal{I}(B):= \mathbb{E}[(P^{'}_{B}(B) /P_{B}(B))^2]$. We define the \emph{Moreau envelope} $\mathcal{M}_{\psi}(x; \alpha):= \min_{y \in \mathbb{R}} \psi(y) + \frac{1}{2\alpha}(x-y)^2$ and its corresponding \emph{proximal operator} $\text{prox}_{\psi}(x; \alpha):= \argmin_{y \in \mathbb{R}} \psi(y) + \frac{1}{2\alpha}(x-y)^2 $ for $\psi:\mathbb{R}\xrightarrow[]{} \mathbb{R}$ and $\alpha >0$. The first order derivative of the Moreau envelope with respect to $x$ is denoted as $\mathcal{M}^{'}_{\psi,1}(x; \alpha):= \frac{\partial \mathcal{M}_{\psi}(x; \alpha)}{\partial x}$. For a sequence of random variables $\mathcal{X}_n$ that converges in probability to some constant $c$ in the limit according to Assumption \ref{ass: HDA}, we drop the subscript and simply write $\mathcal{X}\xrightarrow[]{P}c$.

\section{Problem Setup}
\label{sec: PS}

We consider the problem of linear regression in the over-parameterized regime. We model the data $\{\bm x_i,y_i\}_{i=1}^{m}$ be generated from an additive noisy linear model:
\begin{equation}
\label{eq: linear regression}
y_i = \bm{x_i}^{T}\bm{\beta}^{*} + z_i \quad \text{for} \quad i=1,\cdots,m
\end{equation}
where labels/outputs $y_i \in \mathbb{R}$ are a linear function of the covariate vectors $\bm{x}_i \in \mathbb{R}^n$ perturbed by unknown noise $z_i \in \mathbb{R}$. Here, $\bm{\beta}^{*} \in \mathbb{R}^n$ is the true unknown model/weights to be estimated through learning. The goal of the learner is to come up with an estimate $\bm{\hat \beta}$ that minimizes the population risk/generalization error
\begin{equation}
\label{eq: ge}
    r(\bm{\hat \beta}):= \mathbb{E}_{\Tilde{y},\Tilde{\bm x}}[(\Tilde{y} - \Tilde{\bm x}^{T}\bm{\hat \beta})^2]
\end{equation}
or equivalently the excess risk $r(\bm{\hat \beta}) - r(\bm{\beta}^*)$. Here, the expectation is over an independent realization of the $(\Tilde{y},\Tilde{\bm x})$, which are related by \eqref{eq: linear regression}. In the overparameterized regime $n>m$, the minimizer obtained from linear regression on the dataset is typically not unique due to interpolation, and we generally have a subspace of global optima. The goal of this work is to study the dependence of the generalization error on the structure of the underlying signal $\bm\beta^*$, covariates, and the choice of global optima from the interpolating subspace. To this end, we assume the following in our analysis.
\begin{assumption}($\text{High-dimensional asymptotics}$) 
\label{ass: HDA}
Throughout this paper, we consider the asymptotic proportional limit where both $m,n \xrightarrow[]{} \infty$ at a fixed ratio $\delta = m/n$, where $0<\delta<1$.
\end{assumption}


\begin{assumption}($\text{Gaussian features and noise}$)
\label{ass: Gaussian feat and noise}
Given the feature covariance $\bm \Sigma$ the data/feature vectors $\bm x_i: = \bm \Sigma^{\frac{1}{2}}\bm g_i$ where $\bm g_i  \overset{i.i.d.}{\sim} \mathcal{N}(\bm 0, \frac{1}{n}\bm I_n)$, $i \in[m]$ and the noise $z_i \overset{i.i.d.}{\sim} \mathcal{N}(0,\sigma^2),\; i \in [m]$. 
\end{assumption}

\begin{assumption}($\text{True parameter and eigen-spectrum distribution}$)
\label{ass: true dist}
We consider a fixed diagonal covariance $\bm \Sigma$ and unknown signal parameter $\bm \beta^{*}$ such that the pair $\{\Sigma_{i, i},\beta^{*}_i\}$ of the eigen-values and the coordinates of the true underlying signal parameter are sampled i.i.d. from $\{\Lambda^2, B\}$ with distribution $P_{\Lambda, B}$. Additionally, we assume that the marginal distribution $P_B$ has a finite, non-zero second moment and $\Lambda$ is a strictly positive with bounded support.

\end{assumption}

The requirement of Gaussianity in the covariates and noise is a bit conservative and is more of a technical requirement. We believe our results only depend on the second-order statistics of the data and noise due to Gaussian universality for sub-Gaussian covariates and noise using arguments from \cite{bayati2015universality,oymak2018universality,han2023universality,ghane2024universality}. Under these assumptions, the generalization error in \eqref{eq: ge} simplifies to the following limit.
\begin{equation}
    r(\bm{\hat \beta}) = \lim_{n\rightarrow\infty} \frac{1}{n}(\bm{\hat \beta} - \bm{\beta}^*)^{T}\bm \Sigma(\bm{\hat \beta} - \bm{\beta}^*) + \sigma^2.
\end{equation}
We denote $\bm{X}:= (\bm{x}_1,\bm{x}_2,\cdots,\bm{x}_m)^{T} \in \mathbb{R}^{m\times n}$, label vector $\bm{y}:= (y_1,y_2,\cdots,y_m)^{T}\in \mathbb{R}^m$ and $\bm{z} := (z_1, z_2, \cdots, z_m)^{T} \in \mathbb{R}^m$. We assume that the data and noise satisfy the conditions in Assumption \ref{ass: Gaussian feat and noise} and that $\bm \Sigma$ is fixed. Performing empirical risk minimization (ERM) on the dataset in the overparameterized regime leads to interpolation, and we denote manifold of global optima as $\Tilde{\mathcal{B}}:= \{\bm \beta: \bm y = \bm X \bm \beta\}$. The global optimum we converge to is given by the implicit bias of the algorithm, and in this work, we consider separable convex potentials, i.e. $\Psi(\bm \beta):= \sum_{i=1}^{n}\psi_i(\beta_i)$ which satisfy $\lim_{\|\bm \beta\| \rightarrow \infty} \Psi(\bm \beta) = \infty$. Further, we consider the general case of $\psi_i(.) = \psi(.,\Sigma_{i,i})$ where the learner has access to the diagonal entries of $\bm \Sigma$. For cases where $\bm\Sigma$ is unknown to the learner, one can simply take $\Psi(\bm \beta):= \sum_{i=1}^{n}\psi(\beta_i)$. Our analysis is general enough to cover both these cases. Under these assumptions, \eqref{eq: implicit bias} can be reformulated as
\begin{equation}
\label{eq: sep implicit bias}
 \bm{\hat \beta}:= \argmin_{\bm \beta \in \mathbb{R}^n} \sum_{i=1}^{n}\psi(\beta_i,\Sigma_{i,i}) \quad \text{s.t} \quad \bm y = \bm X \bm \beta.  
\end{equation} 
This formulation includes the minimum $\ell_2$ norm interpolator as a special case when $\Psi(\bm \beta) = \|\bm \beta\|_2^2$. When $\bm \Sigma$ is diagonal, we can also pick $\Psi(\bm \beta) = (\bm{\beta} - \bm{\beta}^*)^{T}\bm \Sigma(\bm{\beta} - \bm{\beta}^*)$ which recovers the theoretical optimal interpolator shown in \cite{muthukumar2020harmless}, although this is not achievable since $\bm{\beta}^*$ is unknown to the learner.

\subsection{Contributions}
Our contributions are summarized as follows
\begin{itemize}
    \item \textbf{Sharp asymptotics.} We characterize the precise asymptotic limit of the generalization error of interpolators obtained through \eqref{eq: implicit bias} for separable $\Psi$; see Theorem \ref{thm: PA}. We show that the value of this limit is obtained as a solution of a system of two nonlinear equations with two unknowns. We provide sufficient conditions for the existence and uniqueness of these solutions and the class of convex potentials for which these conditions hold. Additionally, we also provide the distributional characterization of the weights of the interpolating solutions obtained. This is significant as it facilitates the analysis of generalization errors after post-training operations like pruning and quantization. We also provide a detailed analysis for the special cases of the minimum $\ell_1$,$\ell_2$,$\ell_3$ and $\ell_{\infty}$-norm interpolators in Appendix \ref{apx: PA}. The main analysis technique used involves Gaussian comparison lemmas, which also allows for non-asymptotic analysis, which we defer for future work.
     
    \item \textbf{Fundamental limits.} We establish a tight lower bound on the achievable generalization error for a broad class of interpolators obtained through \eqref{eq: implicit bias}, and computing this lower bound involves solving a scalar non-linear equation; see Theorem \ref{thm: lower bound}. We also provide a slightly weaker but simplified version of this lower bound for isotropic data, and we show that this bound is indeed tight when the true underlying parameter has a Gaussian density (Corollary \ref{cor: closed lower bound}). 
    
    \item \textbf{Optimal implicit bias.} Under certain conditions, we provide a construction of the optimal convex potential, whose asymptotic limit of the generalization error matches the lower bounds, indeed confirming that the lower bounds are tight (Theorem \ref{thm: optimal potential}). We also describe the special case when SGD or the minimum $\ell_2$-norm interpolator is optimal (Corollary \ref{cor: optimal for Gaussian}). 
    
    \item \textbf{Numerical simulations.} In section \ref{sec: numerical sim}, we numerically evaluate our results for some special cases of convex potentials when the distribution of the true underlying unknown parameter is a sparse Gaussian, Rademacher, and Gaussian.
\end{itemize}

\subsection{Related Work}

There has been extensive literature studying the precise asymptotics of different convex regularized estimators for linear models \cite{donoho2011noise, stojnic2009various, bayati2011lasso, chandrasekaran2012convex, amelunxen2013living, oymak2016sharp, abbasi2016general, stojnic2013framework, oymak2013squared, thrampoulidis2015regularized, karoui2013asymptotic, donoho2016high, el2018impact, thrampoulidis2018precise,oymak2018universality,dobriban2018high, lei2018asymptotics, miolane2021distribution, hastie2022surprises, wang2022does, celentano2022fundamental, hu2022slope, bu2019algorithmic, emami2020generalization, lolas2020regularization, chang2021provable}. The convex Gaussian minimax Theorem (CGMT) \cite{stojnic2013framework, thrampoulidis2015regularized} provides a framework to do this precise asymptotic analysis in many of these afore-mentioned works and will also be the primary tool for analysis in our work. Another popular approach used in precise asymptotic analysis is approximate message passing (AMP) \cite{donoho2009message,donoho2011noise}, and it is conceivable that results obtained in our work can also be derived using AMP analysis. In the context of interpolation, under the proportional regime of Assumption \ref{ass: HDA} \cite{hastie2022surprises,chang2021provable} do a precise asymptotic analysis of the minimum $\ell_2$-norm interpolators and \cite{li2021minimum} studied the precise asymptotics of the minimum $\ell_1$-norm interpolator for isotropic Gaussian data using AMP. Our results extend this analysis to general separable convex potentials on non-isotropic Gaussian data, which include $\ell_1$ and $\ell_2$ norms as special cases, and we recover the previous results.

There is a rich body of work on non-asymptotic analysis showing consistent rates of minimum-norm interpolators \cite{bartlett2020benign, zhou2020uniform, negrea2020defense, koehler2021uniform, donhauser2022fast, wang2022tight, chatterji2022foolish, kurminimum}. These works typically consider the highly overparameterized regime, i.e., $\delta \ll 1$, which is a necessary condition for consistency. In contrast, we consider the proportional regime, where $m,n \rightarrow \infty$ and $m/n \rightarrow \delta$, where consistency is not possible \cite{muthukumar2020harmless}, and therefore a sharp characterization of the asymptotic generalization error is of interest. The recent works of \cite{zhou2020uniform,negrea2020defense,koehler2021uniform,wang2022tight} also use CGMT in their analysis, where Gaussian comparison lemmas are used to obtain bounds on the risk using uniform convergence arguments. This approach is different from ours, where CGMT is used directly on the objective \eqref{eq: implicit bias}, which simplifies the objective to a scalar optimization in the asymptotic limit similar to \cite{chang2021provable}.

In terms of characterizing fundamental limits, \cite{bean2013optimal, donoho2016high, advani2016statistical} were the first works to derive lower bounds and optimal loss functions for high-dimensional linear regression problems in the absence of regularization, and therefore, they consider the under-parameterized regime with a unique global optimum. More recently, \cite{celentano2022fundamental} studied the fundamental limits of convex regularized linear regression, where they considered the square loss and derived lower bounds on the prediction error obtained from an optimally tuned convex regularizer. Similar results on the lower bounds of the prediction error were also studied for binary classification \cite{taheri2020sharp} and for ridge-regularized regression for linear and binary models in \cite{taheri2021fundamental}. None of these prior works consider the case of interpolation in over-parameterized models, and our analysis extends these ideas to derive fundamental lower bounds on the generalization error of interpolating solutions on linear models obtained from minimizing convex potentials, improving upon previous results of \cite{amari2020does,oravkin2021optimal}.

\section{Main Results}
In this section, we provide the main results of our work. We first characterize the precise asymptotic of problem \eqref{eq: sep implicit bias} in terms of the solution to a system of non-linear equations \eqref{eq: KKT}. Next, we leverage this system of equations to drive tight lower bounds on the generalization error and provide the construction of the optimal convex potential that achieves these bounds.

\subsection{System of Non-Linear Equations}
For a given overparameterization ratio $0<\delta<1$ and noise variance $\sigma^2$, we have the following system of non-linear equations in $\alpha,u$
\begin{subequations}
 \label{eq: KKT}
 \begin{align}
\mathbb{E}\left[\frac{H}{\Lambda} \mathcal{M}_{\psi,1}^{'}(B+\frac{\alpha}{\sqrt{\delta}\Lambda}H; \frac{\alpha}{u\sqrt{\delta}\Lambda^2})\right] &= u(1-\delta) \\
\mathbb{E}\left[\left(\frac{1}{\Lambda} \mathcal{M}_{\psi,1}^{'}(B+\frac{\alpha}{\sqrt{\delta}\Lambda}H; \frac{\alpha}{u\sqrt{\delta}\Lambda^2})\right)^2\right] &= u^2(1-\delta) - \frac{\delta \sigma^2 u^2}{\alpha^2}
 \end{align}
\end{subequations}
Here the expectation is over the random variables $B, H$ and $\Lambda$, where $\Lambda, B$ as defined in Assumption \ref{ass: true dist} denote the distribution of the eigenspectrum and the underlying signal respectively and $H$ is a standard Gaussian. Next, we state Theorem \ref{thm: PA}, which relates the generalization error of \eqref{eq: sep implicit bias} to the solution of the above system of equations.

\subsection{Precise Asymptotics}
\label{sec: PA}

\begin{theorem}(Generalization error)
\label{thm: PA}
Let Assumptions \ref{ass: HDA}, \ref{ass: Gaussian feat and noise}, and \ref{ass: true dist} hold and $\bm{\hat \beta}$ be the solution of \eqref{eq: sep implicit bias}. if we assume that $\alpha_{\psi},u_{\psi}$ are the unique solutions to \eqref{eq: KKT}, then in the asymptotic limit $n,m\rightarrow \infty, \frac{m}{n} \rightarrow \delta$, we have that
\begin{equation}
\label{eq: asym ge}
    r(\bm{\hat \beta}) \xrightarrow[]{P} \alpha_{\psi}^2.
\end{equation}
\end{theorem}

The statements in Theorem \ref{thm: PA} hold for a general class of separable convex potentials in \eqref{eq: sep implicit bias}. In the special cases of $\ell_1$ and $\ell_2$ potentials, our analysis recovers many of the results from \cite{hastie2022surprises,chang2021provable,li2021minimum}. We also do an in-depth analysis of the special cases of the minimum $\ell_1$,$\ell_2$,$\ell_3$ and $\ell_{\infty}$ norm interpolators using a slightly different (but equivalent) approach, the details of which are deferred to the Appendix \ref{apx: PA}. We additionally note that although Theorem \ref{thm: PA} provides an asymptotic characterization, the application of CGMT, which is a comparison lemma, allows for non-asymptotic analysis of the generalization error, which is beyond the scope of this work.

\begin{remark}(Uniqueness of $\alpha_{\psi}$ and $u_{\psi}$)
The solution $\alpha_{\psi}$ to the system of nonlinear equations \eqref{eq: KKT}, which is the primary quantity of interest as always exists and is unique. To guarantee the existence of a unique $u_{\psi}$, it is sufficient to assume that the Moreau envelope
is strictly concave in $u$, which is true when $\psi$ is strictly concave.
\end{remark}


\begin{remark}(Distributional Convergence)
\label{re: DC}
Additionally, we have convergence of the joint empirical distribution of $\bm{\hat\beta}$ in the limit such that
\begin{equation}
    \frac{1}{n}\sum_{i=1}^{n}f(\bm{\hat \beta}_i,\bm \beta^*_i,\Sigma_{i,i}) \xrightarrow[]{P} \mathbb{E}[f( Z_{\delta,\sigma^2},B,\Lambda^2)],
\end{equation}
holds for every bounded Lipschitz function $f$, where $Z_{\delta,\sigma^2} := \text{prox}_{\psi}(B + \frac{\alpha_{\psi}}{\sqrt{\delta}\Lambda} H ; \frac{\alpha_{\psi}}{u_{\psi}\sqrt{\delta}\Lambda^2})$.  
\end{remark}
Remark \ref{re: DC} characterizes the asymptotic joint empirical distribution of the tuple $(\bm{\hat \beta}_i,\bm \beta^*_i,\bm \Sigma_{i,i})$. The distributional analysis is useful because it helps us study the structure of the weights and the performance after post-processing of the solution weights (like pruning and quantization).

\subsection{Fundamental Limits}

\label{sec: opt imp bias}

In this section, we study the fundamental limits on the generalization error of interpolators obtained through \eqref{eq: sep implicit bias}. To be precise, we consider the following class of convex potentials defined below
\begin{equation}
\label{eq: interpolation class}
\mathcal{C}_{\psi}:= \left\{ \Psi \mid \Psi(\bm \beta) = \sum_{i=1}^{n}\psi(\beta_i, \Sigma_{i,i}) \quad\text{  s.t   $\quad \psi(.,\Sigma_{i,i})$ is convex $\forall i \in [n]$} \right\}.\
\end{equation}
Therefore $\mathcal{C}_{\psi}$ contains a much broader class of separable convex functions that can have a dependence on the eigenspectrum of the data source, which is assumed to be known a priori. We provide a lower bound on the generalization error, which is valid for every potential in $\mathcal{C}_{\psi}$. Additionally, under certain conditions, we show that this lower bound is tight by constructing the optimal convex potential, which achieves this bound. 

\begin{theorem}(Lower bound on $\alpha_{\psi}^2$)
\label{thm: lower bound}
Let Assumptions \ref{ass: HDA},\ref{ass: Gaussian feat and noise}, and \ref{ass: true dist} hold. Define the random variable $V_{\alpha} = B + \frac{\alpha}{\sqrt{\delta}\Lambda}H$ where $H\sim \mathcal{N}(0,1)$ and $B,\Lambda$ as defined in Assumption \ref{ass: Gaussian feat and noise} and \ref{ass: true dist}. Let $\alpha_{*}$ be the unique solution of the following non-linear equation
\begin{equation}
\label{eq: lower bound}
 \alpha^2= \frac{\delta\sigma^2}{1-\delta}+\frac{\delta(1-\delta)}{\mathcal{I}_{\Lambda}(V_{\alpha})}
\end{equation}
where $\mathcal{I}_{\Lambda}(V_{\alpha})$ is the weighted Fisher information of $V_{\alpha}$ defined as
\begin{equation}
\label{eq: weighted fisher info}
    \mathcal{I}_{\Lambda}(V_{\alpha}) := \mathbb{E}\left[ \left(\frac{\xi_{V_{\alpha}}({V_{\alpha}|\Lambda)}}{\Lambda}\right)^2\right]
\end{equation}
where $\xi_{V_{\alpha}}(v|\Lambda):= \frac{p^{'}_{V_{\alpha}(v|\Lambda)}}{p_{V_{\alpha}(v|\Lambda)}}$ is the conditional score function of $V_{\alpha}$. Then for every $\Psi \in \mathcal{C}_{\psi}$, with $\alpha^2_{\psi}$ as the asymptotic limit of the generalization error as in \eqref{eq: asym ge}, we have $\alpha^2_{\psi} \geq \alpha_{*}^2$.
\end{theorem}
The proof of Theorem \ref{thm: lower bound} is deferred to Appendix \ref{apx: fund lim}, and it also involves showing the existence of a unique solution to the non-linear equation \eqref{eq: lower bound}. Solving \eqref{eq: lower bound} generally involves numerical computation and can be done through fixed-point iterations, and $\mathcal{I}_{\Lambda}(V_{\alpha})$ can be computed through Monte Carlo sampling or numerical integration. The weighted Fisher information \eqref{eq: weighted fisher info} can be computed more generally even when $P_B$ is not a differentiable potential since adding a Gaussian smoothens out the density. Next, we provide a slightly weaker lower bound which avoids solving a non-linear equation like \eqref{eq: lower bound} in the special case of isotropic data, i.e, $\bm \Sigma = \bm I_n$.

\begin{corollary}
\label{cor: closed lower bound}
Let Assumptions \ref{ass: HDA},\ref{ass: Gaussian feat and noise} and \ref{ass: true dist} hold and $\alpha_{*}$ be defined as the solution to \eqref{eq: lower bound}, if $\Lambda = 1$ almost surely then
\begin{equation}
 \label{eq: closed lower bound}
 \alpha_{*}^2 \geq \frac{\sigma^2}{1-\delta} + \frac{1-\delta}{\mathcal{I}(B)}
\end{equation}
whenever, the Fisher information $\mathcal{I}(B)$ is well defined. The inequality becomes an equality if and only if $B$ is a Gaussian.
\end{corollary}
The proof of corollary \ref{cor: closed lower bound} is found in Appendix \ref{apx: fund lim} and involves the application of Stam's inequality for Fisher information to $\mathcal{I}(V_{\alpha})$, which makes it possible to solve \eqref{eq: lower bound} analytically.

Looking at \eqref{eq: closed lower bound} closely, the first term $\frac{\sigma^2}{1-\delta}$ represents the theoretical lower bound on the generalization error for all possible interpolating solutions \cite{muthukumar2020harmless} and it shows that overfitting the noise can benign when $\delta \ll 1$ i.e in the highly over-parameterized regime. The second term $\frac{1-\delta}{\mathcal{I}(B)}$ ends up being equal to the Bayes risk when $B$ is Gaussian, and the variance of the noise is zero. Therefore \eqref{eq: closed lower bound} can be interpreted as the sum of the error from overfitting the noise and the error of the best possible estimator in the absence of noise. Contrary to the first term, $\frac{1-\delta}{\mathcal{I}(B)}$ is minimized when $\delta$ approaches $1$, i.e., we have an equal number of equations and unknowns to fully recover the underlying unknown parameter. So there is an inherent tension between the two terms, and increasing $\delta$, although recovers more of the signal, amplifies the noise due to overfitting.

Next, we will argue that these lower bounds obtained from Theorem \ref{thm: lower bound} and Corollary \ref{cor: closed lower bound} are indeed tight by constructing optimal convex potentials that achieve these lower bounds.

\subsection{Optimal implicit Bias}

\begin{theorem}(Optimal $\Psi$)
 \label{thm: optimal potential}  
Let Assumptions \ref{ass: HDA}, \ref{ass: Gaussian feat and noise}, and \ref{ass: true dist} hold and $\alpha_{*}$ be defined as the solution to \eqref{eq: lower bound}. Consider the following function $\psi_{*} : \mathbb{R}^2 \xrightarrow[]{} \mathbb{R}$
\begin{equation}
 \label{eq: optimal potential}
 \psi_{*}(v,\lambda) :=  - \mathcal{M}_{\log(P_{V_{\alpha_{*}}}(v|\lambda))}\left(v;\frac{\alpha^2_{*}(1-\delta) - \delta\sigma^2}{\delta(1-\delta)\lambda^2}\right),
\end{equation}
if $P_{V_{\alpha_{*}}}(v|\lambda)$ is log-concave in $v$ and we define $\Psi_{*}(\bm \beta) = \sum_{i=1}^{n}\psi_{*}(\beta_i, \Sigma_{i,i})$, then
\begin{enumerate}
\item $\Psi_{*}(\bm \beta) \in \mathcal{C}_{\psi}$
\item $\alpha_*$ is a solution to the system of equations \eqref{eq: KKT} obtained using $\psi^*(v,\lambda)$ and is therefore the optimal convex implicit bias.
\end{enumerate}
\end{theorem}

Theorem \ref{thm: optimal potential} provides a construction of the optimal potential \eqref{eq: optimal potential} that satisfies the system of equations \eqref{eq: KKT}. When $\Psi_{*}$ obtained using \eqref{eq: optimal potential} belongs to $\mathcal{C}_{\psi}$, then the $\alpha^2_{\psi_{*}}$ obtained from \eqref{eq: KKT} denotes the asymptotic limit of the generalization error of $\Psi_{*}$ and by Theorem \ref{thm: optimal potential}, we have that $\alpha^2_{\psi_{*}} = \alpha^2_{*}$, thereby achieving the lower bound. The proof of this deferred to Appendix \ref{apx: fund lim} and involves verifying that the construction in \eqref{eq: optimal potential} satisfies the system of equations \eqref{eq: KKT} and characterizing the sufficient conditions for the convexity of $\Psi_{*}$. 
\begin{remark}
The log-concavity of $P_{V_{\alpha_{*}}}(v|\lambda)$ must be verified on a case-by-case basis by first solving for $\alpha_{*}$. A sufficient condition that always ensures log-concavity of $P_{V_{\alpha_{*}}}(v|\lambda)$ is by letting $P_{B}$ be a log-concave density since convolving a Gaussian density with $P_{B}$ preserves its log-concavity. Even when $P_{B}$ is not log-concave, it is possible that if $\alpha_{*}$ obtained from solving \eqref{eq: lower bound} is greater than a certain threshold value, it smoothens out $P_B$ enough such that the resulting $P_{V_{\alpha_{*}}}(v|\lambda)$ is log-concave. This, in turn, characterizes a region of $\delta$ and $\sigma^2$, where the optimal convex potential is achievable and the lower bound is tight.
\end{remark}
We'll observe in numerically in Section \ref{sec: numerical sim}, where we consider the cases when $P_B$ is sparse-Gaussian and Rademacher both of which are not even differentiable but the optimal potential construction obtained from Theorem \ref{thm: optimal potential} is indeed convex for the values of $\delta$ and $\sigma^2$ chosen. Next, we look at the special case when the underlying parameter density $P_{B}$ is a Gaussian density. 

\begin{corollary}($\Psi_{*}$ for Gaussian $B$)
 \label{cor: optimal for Gaussian}  
 Let Assumptions \ref{ass: HDA}, \ref{ass: Gaussian feat and noise}, and \ref{ass: true dist} hold and $B\sim \mathcal{N}(0,1)$, then the optimal implicit bias is given as 
 \begin{equation}
  \label{eq: optimal for Gaussian}
  \Psi_*(\bm \beta) = \bm \beta^{T}\bm\Sigma^{1/2}\left(\frac{\sigma^2}{1-\delta} \bm I_{n} + \bm\Sigma\right)^{-1}\bm\Sigma^{1/2} \bm \beta.
 \end{equation}
\end{corollary}

Corollary \ref{cor: optimal for Gaussian} comes from the direct application of Theorem \ref{thm: optimal potential} while ignoring the constant terms in the potential since they don't affect the outcome of the optimization. In fact, we observe that \eqref{eq: optimal for Gaussian} is quadratic, which is convex, and therefore, it achieves the lower bound $\alpha _ *^2$ on the generalization error. If $\frac{\sigma^2}{1-\delta} \gg \Lambda$, which happens either when the noise variance $\sigma^2$ is large or we are slightly over-parameterized, i.e., $\delta$ approaches 1, then the optimal potential $\Psi_{*}(\bm \beta) \approx \bm \beta^{T}{\bm \Sigma}\bm \beta$. In the other case, when $\frac{\sigma^2}{1-\delta} \ll \Lambda$, the optimal potential is $\Psi_{*}(\bm \beta) \approx \|\bm \beta\|^2_{2}$. In the special case when $\Sigma$ is isotropic or when the variance of the noise $\sigma^2=0$, the optimal convex potential is exactly the Euclidean norm squared, the implicit bias of SGD.
\begin{remark}(Comparison with \cite{oravkin2021optimal})
Although we consider a bigger class of interpolators in \eqref{eq: interpolation class}, the set of interpolators considered in \cite{oravkin2021optimal} is not completely inclusive in our class. The key difference is that our optimal convex potential construction doesn't depend on the observed labels data $X,y$. In contrast, the optimal linear response interpolator considered in \cite{oravkin2021optimal} involves a quadratic potential that can depend on $X,y$, and this dependence makes analysis quite difficult outside of the quadratic case. This seems to be a subtle difference, and we, in fact, see similar qualitative trends as discussed under Corollary \ref{cor: optimal for Gaussian}. 
\end{remark}

\begin{figure*}[h]
    \centering
    \begin{tabular}{ccc}
    \includegraphics[width=.3\textwidth]{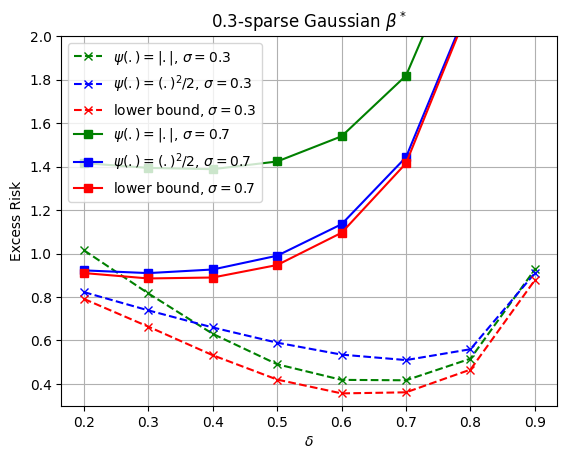}&
    \includegraphics[width=.3\textwidth]{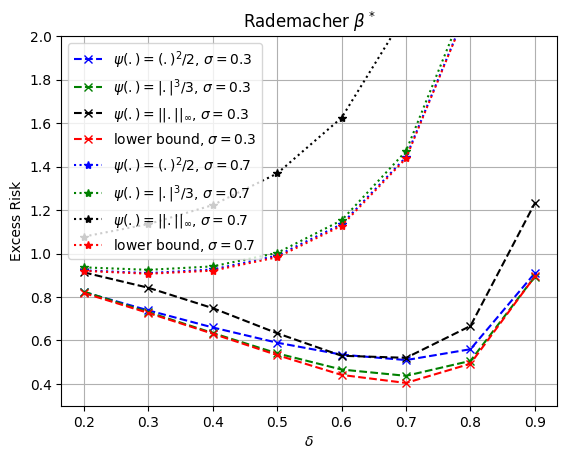}&
    \includegraphics[width=.3\textwidth]{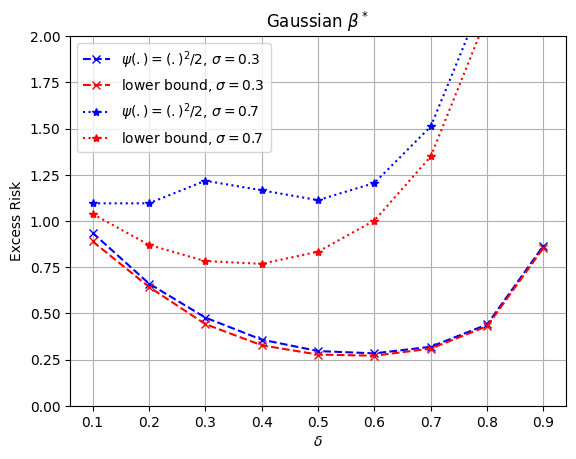}\\
    (a)&(b)&(c)
    \end{tabular}
    \caption{Comparison of excess risk of different interpolators (a) when $P_{B}$ has a 0.3-sparse Gaussian density and $\Lambda = 1$ a.s (b) when $P_{B}$ has a Rademacher density and $\Lambda = 1$ a.s (c) when $P_{B}$ has a Gaussian density and $\Lambda^2$ takes the value $4$ with probability $0.3$ and $0.1$ with probability $0.7$.}
     \label{fig: ER}
\end{figure*}
\section{Numerical Simulations}
\label{sec: numerical sim}

 In this section, we provide numerical simulations of the results derived in section \ref{sec: PA} and \ref{sec: opt imp bias} and provide insights on the implications of these results for certain special cases. In particular, we study the cases when the prior distribution of the underlying signal parameter $P_B$ is a sparse Gaussian (Figure \ref{fig: ER} (a)), Rademacher (Figure \ref{fig: ER} (b)) and Gaussian (Figure \ref{fig: ER} (c)). Figures \ref{fig: opt_spar} (a) and \ref{fig: opt_spar} (b) show the structure of the optimal convex potentials when the underlying signal parameter $P_B$ is a sparse Gaussian and Rademacher, respectively. We normalize the underlying prior signal such that $\mathbb{E}[B^2] = 1$ and define the signal-to-noise ratio (SNR) as $\frac{1}{\sigma^2}$ and consider the regimes of low SNR ($\sigma=0.7$) and high SNR ($\sigma=0.3$). All the above-mentioned plots evaluate the asymptotic theoretical limits of the results derived and are obtained from solving a system of nonlinear equations involving expectations of certain quantities. We use standard packages like Scipy to compute these expectations using numerical integrals and solve the system of non-linear equations. Additional plots demonstrating the concentration of non-asymptotics are deferred to the Appendix \ref{apx: PA}.
 
\begin{figure}
    
    \centering
    \begin{tabular}{cc}
    \includegraphics[width=.45\textwidth]{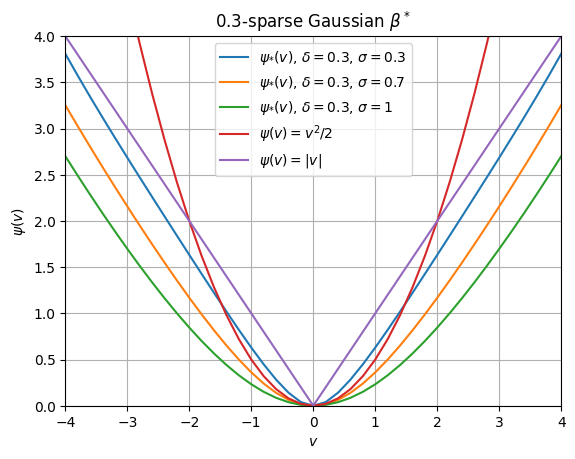}&
    \includegraphics[width=.45\textwidth]{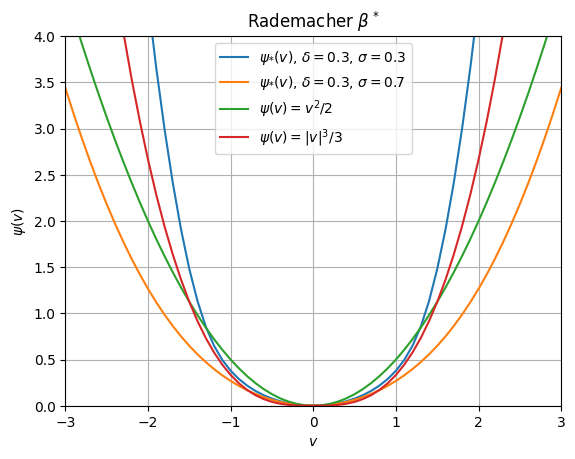}\\
    (a)&(b)
    \end{tabular}

    \caption{Comparison of the structure of the optimal convex potentials when (a) $P_{B}$ has a 0.3-sparse Gaussian density and $\Lambda = 1$ a.s (b) $P_{B}$ has a Rademacher density density and $\Lambda = 1$ a.s}
    \label{fig: opt_spar}
\end{figure}

    

 \textbf{Sparse Gaussian.} In Figure \ref{fig: ER} (a), we consider the case when the underlying unknown parameter $B$ is 0.3-sparse Gaussian, i.e., with probability $0.3$ behaves like a Gaussian and is zero otherwise. We consider isotropic data and $B$ is scaled appropriately such that its variance is one. The y-axis represents the excess risk of the interpolating solution obtained, and the x-axis sweeps across the overparameterization ratio $\delta$. In the absence of noise, it's well known that we get perfect recovery for the $\ell_1$ norm at approximately two times the sparsity of the signal \cite{chandrasekaran2012convex}, but in the presence of noise, interpolation can't recover the signal. In the high SNR regime ($\sigma = 0.3$), we observe that the minimum $\ell_1$ interpolator does, in fact, outperform the minimum $\ell_2$ interpolator for certain regions of $\delta$. But in the low SNR regime ($\sigma = 0.7$), observe that the $\ell_1$ interpolator performs significantly worse than $\ell_2$, which suggests that imposing structure while interpolating noisy labels can amplify the noise more than the recovery of the signal. We additionally observe that $\ell_2$ is, in fact, very close to the optimal performance characterized by the lower bound. Figure \ref{fig: opt_spar} (a) shows the structure of the optimal convex potentials that achieve these lower bounds at $\delta = 0.3$ at different SNRs. We see that the optimal potential behaves like a smoothened version of the $\ell_1$-norm, and as SNR decreases, the optimal potential approaches the Euclidean norm squared, supporting our previous observations that the $\ell_2$-norm interpolator is, in fact, close to optimal at low SNR.

\textbf{Rademacher.} In Figure \ref{fig: ER} (b), we consider Rademacher distributed $B$, where it take the values $\{\pm 1\}$ with equal probability with isotropic Gaussian data. The recovery threshold in the absence of noise for the $\ell_{\infty}$-norm was shown to be $\delta = 0.5$ \cite{chandrasekaran2012convex}. In the presence of noise, even in the high SNR regime ($\sigma=0.3$), we observe that $\ell_{\infty}$-norm is outperformed by both the minimum $\ell_2$ and $\ell_3$ norm interpolators. In the low SNR regime ($\sigma=0.7$), this gap in performance only grows wider as expected. Looking at the structure of the optimal convex potential in Figure \ref{fig: opt_spar} (b), we see that for high SNRs, the optimal potential is much flatter around the origin and increases steeply at around $1$ compared to the square and cubic potentials and as we move to the low SNR regime, the optimal potential smoothens out.

\textbf{Gaussian.} Finally, we consider the Gaussian prior in Figure \ref{fig: ER} (c). For isotropic data, we have established in Corollary \ref{cor: optimal for Gaussian} that the optimal potential is indeed the minimum $\ell_2$-norm interpolator. So now, we consider non-isotropic data with a bi-level eigen-spectrum where the diagonal entries of covariance $\bm \Sigma$ are set to $4$ with probability $0.3$ and $0.1$ with probability $0.7$. In the high SNR regime ($\sigma=0.3$), we observe that the minimum $\ell_2$-norm interpolator is quite close to the lower bound, which is achieved by the optimal potential, which has access to the elements of $\bm \Sigma$. This is not surprising since from \eqref{eq: optimal for Gaussian}, we observe that for high SNRs, the optimal potential approaches the Euclidean norm squared. In contrast, for the low SNR regime ($\sigma=0.7$), we see a significant gap in the performance of the minimum $\ell_2$ norm interpolator and the optimal achievable interpolator. This again can be explained from \eqref{eq: optimal for Gaussian}, where for low SNRs, we see that the optimal convex potential $\Psi_{*}(\bm \beta) \approx \bm \beta^{T}{\bm \Sigma}\bm \beta$. Therefore, having access to the eigen-spectrum of the data source makes the difference when the noise variance is large. 
\section{Deep Neural Networks}
\label{sec:DNN}
In this section, we discuss the applicability of our results to non-linear models, in particular deep networks. Linearity plays a key role throughout our analysis; therefore, our results presented may not directly translate to general non-linear settings, but in certain regimes, it can be shown that neural networks are well-approximated by linear models. The neural tangent kernel (NTK) \cite{jacot2018neural} framework has been one the main tools to theoretically understand the optimization of infinitely wide neural networks. In this infinite width limit, training using gradient descent becomes equivalent to optimizing linear functions in the infinite-dimensional Hilbert space defined by the NTK, which is in line with our problem setting of letting $n\rightarrow\infty$. Another key assumption we make is the Gaussianity of the data. Although it is conceivable that due to Gaussian universality, the analysis shown holds more generally for non-Gaussian data, this remains to be shown. In terms of a practical algorithm to arrive at these minimum convex interpolators, \cite{azizan2021stochastic} shows both experimentally and theoretically the validity of the implicit regularization property of SMD if the initialization is close enough to the manifold of global minima (something that comes for free in the highly overparameterized case). The optimal convex potentials derived in this work can be implemented directly using the SMD update rule if the derivative of the convex potential is invertible, and the separability makes this implementation efficient. Clearly, the choice of potential plays an important role in determining the generalization performance, as seen in Figure \ref{fig: ge cifar}; therefore, it should be treated as a hyper-parameter during optimization. An extensive survey of empirical experiments by varying the choice of potential on different model architectures and problem domains would be useful in guiding the choice of potential for a given learning problem.

\section{Conclusion}
This work provides a precise asymptotic analysis of the generalization performance interpolators for linear models obtained as a minimization of a convex potential, which is characterized by the implicit bias of the optimization algorithm. Additionally, we also derive the fundamental lower bounds on the achievable generalization error of interpolators obtained from the minimization of convex potentials and characterize the optimal convex potential that achieves these bounds. Extending these results to non-asymptotic settings and characterizing the optimal implicit bias in this setting are important future directions. It would also be interesting to generalize these results to non-linear models and even study the role of implicit bias in classification problems.

\bibliographystyle{apalike}  
\bibliography{ref}

\newpage 
\appendix


\section{Useful facts}
\label{apx: background}

\subsection{Properties of Moreau Envelope}
In this section, we provide some properties of Moreau Envelope, which will be used in our analysis later. These results are mainly borrowed from \cite{rockafellar2009variational} and also previously used in \cite{taheri2020sharp}
\begin{proposition}
\label{prop: ME}   
Let $\psi : \mathbb{R}\xrightarrow[]{}\mathbb{R}$ be a proper lower semi-continuous convex function. 
\begin{enumerate}
    \item  Then $\lim_{\alpha \xrightarrow[]{} 0_{+}}\mathcal{M}_{\psi}(x; \alpha) \xrightarrow[]{} \psi(x)$ and $\lim_{\alpha \xrightarrow[]{} \infty}\mathcal{M}_{\psi}(x; \alpha) \xrightarrow[]{} \min_{y\in \mathbb{R}}\psi(y)$.
    \item The derivatives of the Moreau Envelope satisfy the following
    \begin{equation}
    \mathcal{M}^{'}_{\psi,1}(x; \alpha):= \frac{\partial \mathcal{M}_{\psi}(x; \alpha)}{\partial x} = \frac{1}{\alpha}(x - \text{prox}_{\psi}(x; \alpha)),
    \end{equation}
    \begin{equation}
    \mathcal{M}^{'}_{\psi,2}(x; \alpha):= \frac{\partial \mathcal{M}_{\psi}(x; \alpha)}{\partial \alpha} = -\frac{1}{2\alpha^2}(x - \text{prox}_{\psi}(x; \alpha))^2 .
    \end{equation}
\end{enumerate}
\end{proposition}

\begin{proposition}
 \label{prop: ME FC}
 For $\alpha>0$ and a function $h$, we have that
 \begin{equation}
     \mathcal{M}_{\psi}(x; \alpha) = \frac{x^2}{2\alpha} - \frac{1}{\alpha}\left(\frac{x^2}{2} + \alpha \psi(x)\right)^{*}
 \end{equation}
 where $\left(\frac{x^2}{2} + \alpha \psi(x)\right)^{*}$ is the convex conjugate of $\frac{x^2}{2} + \alpha \psi(x)$.
\end{proposition}

\begin{proposition}{(Inverse of Moreau envelope)}{[\cite{advani2016statistical}, result 23 in appendix]}
\label{prop: ME inv}
    For $\alpha>0$ and $\psi$ a convex, lower semi-continuous function such that $g(x) = \mathcal{M}_{\psi}(x; \alpha)$, the Moreau envelope can be inverted so that $\psi(x)= - \mathcal{M}_{-g}(x; \alpha)$.
\end{proposition}

\subsection{Properties of Fisher Information}
We now state some standard properties of the Fisher information of location which are used in our analysis and the additional details of which can be found in \cite{blachman1965convolution}. 

\begin{proposition}
\label{prop: FI}
Let $X$ be a zero mean random variable with a differentiable probability density $P_{X}$ such that $P_{X}(x)>0, -\infty<x<\infty$ and the following integral is well-defined
\begin{equation}
    \mathcal{I}(X) := \int_{-\infty}^{\infty} \frac{(P^{'}_{X}(x))^2}{P_{X}(x)}dx.
\end{equation}
The Fisher information of location $\mathcal{I}(X)$ defined as above satisfies the following properties.
\begin{enumerate}[label=(\alph*)]
    \item For any $c\in \mathbb{R}$, $\mathcal{I}(cX) = \mathcal{I}(X)/c^2$
    \item (Cramer-Rao bound) $\mathcal{I}(cX)\leq \frac{1}{\mathbb{E}[X^2]}$, with equality if and only if $X$ has a Gaussian.
    \item Let $X_1,X_2$ be independent random variables with well-defined $\mathcal{I}(X_1),\mathcal{I}(X_2)$ and $\alpha \in [0,1]$. Then it holds that 
\begin{equation}
    \mathcal{I}(X_1+X_2) \leq \alpha^2\mathcal{I}(X_1) + (1-\alpha)^2\mathcal{I}(X_2)
\end{equation}
    \item (Stam's inequality)For the two independent random variables $X_1,X_2$ defined above, the following holds
    \begin{equation}
        \mathcal{I}(X_1+X_2) \leq \frac{\mathcal{I}(X_1)\cdot\mathcal{I}(X_2)}{\mathcal{I}(X_1)+\mathcal{I}(X_2)}
    \end{equation}
    The inequality is obtained from optimizing the upper bound over $\alpha$ in (c), and the inequality becomes equality if and only if $X_1 ,X_2$ are independent Gaussians.
\end{enumerate}
\end{proposition}

\begin{proposition}
 \label{prop: FI limits}
 Consider the random variables $H \sim \mathcal{N}(0,1)$ and $B$ such that $V_{\alpha}:= B + \alpha H$ satisfies the conditions in Proposition \ref{prop: FI} for $\alpha \in \mathbb{R}_{>0}$. Then we have that
 \begin{enumerate}[label=(\alph*)]
     \item $\lim_{\alpha\xrightarrow[]{}0_{+}} \alpha^2\mathcal{I}(V_{\alpha}) \xrightarrow[]{} 0$
     \item $\lim_{\alpha\xrightarrow[]{}\infty} \alpha^2\mathcal{I}(V_{\alpha}) \xrightarrow[]{} 1$
 \end{enumerate}
\end{proposition}
The proof of the above proposition involves using property (c) of proposition \ref{prop: FI}, and a similar result was also shown in \cite{taheri2021fundamental}.

\section{Precise Asymptotics}
\label{apx: PA}

In this section, we provide the proof for the precise asymptotics of the interpolating solutions obtained from \eqref{eq: sep implicit bias} and also study the special cases of the minimum $\ell_1$, $\ell_2$, $\ell_3$ and $\ell_{\infty}$-norm interpolator.

\subsection{Proof of Theorem \ref{thm: PA}}


\begin{proof}
Consider the following problem given in \eqref{eq: implicit bias}
\begin{equation}
    \min_{\bm \beta} \Psi(\bm \beta) \quad \text{subject to} \quad \bm y = \bm X \bm \beta
\end{equation}
Doing a change of variable $\bm w:= \frac{1}{\sqrt{n}}\Sigma^{\frac{1}{2}}(\bm \beta - \bm \beta^*)$, we the following constrained minimization problem
\begin{equation}
    \min_{\bm w} \Psi(\bm \beta^* + \sqrt{n}\Sigma^{-1/2}\bm w) \quad \text{subject to} \quad \bm G \bm w = \bm z
\end{equation}
\textbf{Boundedness of solution} Now, we assume that $\bm w$ can be restricted to a large enough bounded set $\bm w \in \mathcal{W}:=\{\bm w \text{ s.t } \|\bm w\|_2 \leq B_{+}\}$ without changing the optimization problem. This is more of a technicality required for the application of CGMT. In \cite{chang2021provable}, it was explicitly shown that for the minimum $\ell_2$-norm interpolator, this assumption is true. Since the value of $\ell_p$-norms is less than $\ell_2$-norm for $p>2$, it can be argued that this bounded set construction for $\ell_2$ is also valid for $\ell_p$-norms bigger than $\ell_2$. But more generally, for separable convex functions, this condition must be verified on a case-by-case basis. Alternatively, if we assume that $\alpha_{\psi}$ is bounded, then letting $B_+ = 2\alpha_{\psi}$ also obviates this issue. Taking this into consideration, we now have the following primary optimization (PO) problem.
\begin{equation}
\Phi(\bm G) = \min_{\bm w \in \mathcal{W}} \Psi(\bm \beta^* + \sqrt{n}\Sigma^{-1/2}\bm w) \quad \text{subject to} \quad \bm G \bm w = \bm z
\end{equation}
Using constrained CGMT formulation \cite{li2020exploring}, the Auxiliary optimization (AO) is given as 
\begin{equation}
\phi(\bm g, \bm h) = \min_{\bm w \in \mathcal{W}} \Psi(\bm \beta^* + \sqrt{n}\Sigma^{-1/2}\bm w) \quad \text{subject to} \quad \|\bm g \| \sqrt{\|\bm w\|^2 + \sigma^2} - \bm h^{T}\bm w - \sigma h \leq 0
\end{equation}
where $\bm g$ and $\bm h$ are random vector with iid standard Gaussian entries and $h$ is an iid standard Gaussian scalar. Bringing the constraint into the objective of the AO using Lagrange multiplier $u \geq 0$ and normalizing the constraint with $\sqrt{n}$, we get
\begin{equation}
( \hat{\bm w}^{AO}_n,u_n) := \text{arg} \min_{\bm w \in \mathcal{W}}\max_{u \geq 0} \psi(\bm \beta^* + \sqrt{n}\Sigma^{-1/2}\bm w) + u(\|\bm g_e\| \sqrt{\delta} \sqrt{\|\bm w\|^2 + \sigma^2} - \bm h_{e}^{T}\bm w - \frac{\sigma h}{\sqrt{n}})
\end{equation}
where $\bm h_e: = \frac{\bm h}{\sqrt{n}}$ and $\bm g_e: = \frac{\bm g}{\sqrt{m}}$. Interchange min max using \cite{fan1953minimax} since the objective is convex in $\bm w$ on a compact set $\mathcal{W}$ and concave in $u$. Next, using the square root trick on $\|\bm g_e\| \sqrt{\|\bm w\|^2 + \sigma^2}$, we have that
\begin{equation}
    \|\bm g_e\| \sqrt{\|\bm w\|^2 + \sigma^2} = \min_{\alpha \in \mathcal{A}}\frac{\alpha \|\bm g_e\|^2 }{2} + \frac{\|w\|^2 + \sigma^2}{2\alpha}.
\end{equation}
where $\mathcal{A}=[\sigma,\sqrt{\sigma^2 + B^2_{+}}]$. Plugging back into the AO, we get
\begin{multline}
( \hat{\bm w}^{AO}_n,u_n,\alpha_n) := \text{arg} \max_{u \geq 0} \min_{\bm w \in \mathcal{W},\alpha \in \mathcal{A}} \frac{u\alpha\sqrt{\delta}\|\bm { g}_e\|^2}{2} + \frac{u\sqrt{\delta}\sigma^2}{2\alpha}\\
- \frac{u\sigma h}{\sqrt{n}} + \frac{u\sqrt{\delta}}{2\alpha}\|\bm w\|^2  -u\bm h_{e}^{T}\bm w +\Psi(\bm \beta^* + \sqrt{n}\Sigma^{-1/2}\bm w)
\end{multline}
Using separability of $\Psi$ and appropriate re-scaling, we let $\Psi (\bm x) = \frac{1}{n} \sum_{i=1}^{n}\psi(x_i, \Sigma_{i,i})$, and moving the minimization over $\bm w$ inside the objective, we get 
\begin{multline}
 (\hat{\bm w}^{AO}_n,u_n,\alpha_n) := \text{arg} \max_{u \geq 0} \min_{\alpha \in \mathcal{A}} \frac{u\alpha\sqrt{\delta}}{2} + \frac{u\sqrt{\delta}\sigma^2}{2\alpha} - \frac{u\sigma h}{\sqrt{n}} \\
 + \min_{\bm w \in \mathcal{W}} \frac{1}{n}\sum_{i=1}^{n}\{ \frac{u\sqrt{\delta}n}{2\alpha}w_i^2  -u\sqrt{n} h_{i}w_i + \psi(\beta_i^* + \sqrt{n}\Sigma_{i,i}^{-1/2} w_i,\Sigma_{i,i})\}.
\end{multline}
Doing a change of variable back to the original weight vector $\bm \beta = \bm \beta^* + \sqrt{n}\Sigma^{-1/2}\bm w$, we get 
\begin{multline}
 (\hat{\bm \beta}^{AO}_n,u_n,\alpha_n) := \text{arg} \max_{u \geq 0} \min_{\alpha \in \mathcal{A}} \frac{u\alpha\sqrt{\delta}}{2} + \frac{u\sqrt{\delta}\sigma^2}{2\alpha} - \frac{u\sigma h}{\sqrt{n}} \\
 + \min_{\bm \beta \in \mathcal{B}_{\beta}} \frac{1}{n}\sum_{i=1}^{n}\{ \frac{u\sqrt{\delta}}{2\alpha}\Sigma_{i,i}(\beta_i - \beta_i^*)^2  -uh_{i}\Sigma_{i,i}^{1/2}(\beta_i - \beta_i^*) + \psi(\beta_i,\Sigma_{i,i})\}  
\end{multline}
where $\mathcal{B}_{\beta}:= \{\bm \beta \text{ s.t } \frac{\sqrt{\Sigma}}{\sqrt{n}}(\bm \beta - \bm \beta^*) \in \mathcal{W}\}$. Completing the squares for $\beta_i$ and identifying the Moreau Envelope gives us $D_n(\alpha,u)$ defined as the optimal value of the objective of the optimization defined below.
\begin{equation}
 (u_n,\alpha_n) := \text{arg} \max_{u \geq 0} \min_{\alpha \in \mathcal{A}} \frac{u\alpha\sqrt{\delta}}{2} + \frac{u\sqrt{\delta}\sigma^2}{2\alpha} - \frac{u\sigma h}{\sqrt{n}} + \frac{1}{n}\sum_{i=1}^{n}\{ \mathcal{M}_{\psi}(\beta_i^* + \frac{\alpha}{\sqrt{\delta}\Sigma_{i,i}^{1/2}}h_i ; \frac{\alpha}{u\sqrt{\delta}\Sigma_{i,i}}) - \frac{u\alpha}{2\sqrt{\delta}}h_i^2\}  
\end{equation}
Here $\hat{\bm \beta}^{AO}_{n,i} = \text{prox}_{\psi}(\beta_i^* + \frac{\alpha_n}{\sqrt{\delta}\Sigma_{i,i}^{1/2}}h_i; \frac{\alpha_n}{u_n\sqrt{\delta}\Sigma_{i,i}})$ is always unique given $\alpha_n, u_n$, since proximal is the solution to a strongly convex optimization. The above optimization is strictly convex in $\alpha$, so the saddle point solutions $ (u_n,\alpha_n)$ have unique $\alpha_n$. For $u_n$ to be unique, we need to assume $\frac{1}{n}\sum_{i=1}^{n}\{ \mathcal{M}_{\psi}(\beta_i^* + \frac{\alpha}{\sqrt{\delta}\Sigma_{i,i}^{1/2}}h_i ; \frac{\alpha}{u\sqrt{\delta}\Sigma_{i,i}})\} $ is strictly concave with probability approaching 1. 

\textbf{Asymptotic limits} We consider, the proportional asymptotic limit $n,m\rightarrow \infty, \frac{m}{n} \rightarrow \delta < 1$. In this limit, $\frac{u\sigma h}{\sqrt{n}} \xrightarrow[]{P} 0 $ and we also have that $\frac{1}{n}\sum_{i=1}^{n}\frac{u\alpha}{2\sqrt{\delta}}h_i^2 \xrightarrow[]{P} \frac{u\alpha}{2\sqrt{\delta}}$. Next, 
\begin{equation}
\frac{1}{n}\sum_{i=1}^{n}\{ \mathcal{M}_{\psi}(\beta_i^* + \frac{\alpha}{\sqrt{\delta}\Sigma_{i,i}^{1/2}}h_i ; \frac{\alpha}{u\sqrt{\delta}\Sigma_{i,i}}) \} \xrightarrow{P} \mathbb{E}[ \mathcal{M}_{\psi}(B + \frac{\alpha}{\sqrt{\delta}\Lambda}H ; \frac{\alpha}{u\sqrt{\delta}\Lambda^2})]
\end{equation}

As a consequence, we have point wise convergence of $D_{n}(\alpha,u) \xrightarrow{P} D(\alpha,u)$ which is the following scalar optimization problem
\begin{equation}
  \text{arg} \min_{\alpha \in \mathcal{A}} \max_{u \geq 0}  D(\alpha,u):= \text{arg} \min_{\alpha \in \mathcal{A}} \max_{u \geq 0} -\frac{u\alpha(1-\delta)}{2\sqrt{\delta}} + \frac{u\sqrt{\delta}\sigma^2}{2\alpha} + \mathbb{E}[ \mathcal{M}_{\psi}(B + \frac{\alpha}{\sqrt{\delta}\Lambda}H ; \frac{\alpha}{u\sqrt{\delta}\Lambda^2})]
\end{equation}
By \cite{andersen1982cox}, since both $D_n$ and $D$ are convex, concave in $\alpha,u$, the convergence is uniform, and we have that the objective of converges
\begin{equation}
    \phi(\bm g, \bm h) \xrightarrow[]{P}  \min_{\alpha \in \mathcal{A}} \max_{u \geq 0}  D(\alpha,u)
\end{equation}
If $\mathbb{E}[ \mathcal{M}_{\psi}(B + \frac{\alpha}{\sqrt{\delta}\Lambda}H ; \frac{\alpha}{u\sqrt{\delta}\Lambda^2})]$ is strictly concave in $u$, then $D(\alpha,u)$ is strictly convex and strictly concave, we have parameter convergence by \cite{newey1994large} (Lemma 7.75), therefore
\begin{equation}
     (\alpha_n,u_n)\xrightarrow{P}(\alpha^*,u^*):=\text{arg} \min_{\alpha \in \mathcal{A}} \min_{u \geq 0}  D(\alpha,u)
\end{equation}
In the absence of strong concavity of $u$, we only have the convergence of $\alpha_n \xrightarrow[]{P} \alpha^*$. Typically, distributional convergence requires strict concavity. Generalization error analysis doesn't.
\begin{equation}
 r(\bm{\hat \beta}^{AO}) = \frac{1}{n}(\bm{\hat \beta}^{AO} - \bm{\beta}^*)^{T}\bm \Sigma(\bm{\hat \beta}^{AO} - \bm{\beta}^*) + \sigma^2 = \| \bm{\hat w}^{AO} \|^2 + \sigma^2 = \|\bm{\bar g}\| \alpha^2_n \xrightarrow[]{P} \alpha_*
\end{equation}
Next, we derive the first-order optimality conditions for the scalar minimax problem.

\textbf{First-order optimality conditions}

\begin{multline}
\frac{\partial D(\alpha,u)}{\partial \alpha} = -\frac{u(1-\delta)}{2\sqrt{\delta}} -\frac{u \sqrt{\delta}\sigma^2}{2\alpha^2} + \frac{1}{\sqrt{\delta}}\mathbb{E}[\frac{H}{\Lambda} \mathcal{M}_{\psi,1}^{'}(B+\frac{\alpha}{\sqrt{\delta}\Lambda}H; \frac{\alpha}{u\sqrt{\delta}\Lambda^2})] \\+  \frac{1}{u\sqrt{\delta}}\mathbb{E}[\frac{1}{\Lambda^2} \mathcal{M}_{\psi,2}^{'}(B+\frac{\alpha}{\sqrt{\delta}\Lambda}H; \frac{\alpha}{u\sqrt{\delta}\Lambda^2})]   = 0
\end{multline}

\begin{equation}
 \frac{\partial D(\alpha,u)}{\partial u} = -\frac{\alpha(1-\delta)}{2\sqrt{\delta}} +\frac{\sqrt{\delta}\sigma^2}{2\alpha} -  \frac{\alpha}{u^2\sqrt{\delta}}\mathbb{E}[\frac{1}{\Lambda^2} \mathcal{M}_{\psi,2}^{'}(B+\frac{\alpha}{\sqrt{\delta}\Lambda}H; \frac{\alpha}{u\sqrt{\delta}\Lambda^2})]  = 0
\end{equation}
Using properties of Moreau Envelope (Proposition \ref{prop: ME}), we have
\begin{equation}
\mathbb{E}[\frac{1}{\Lambda^2} \mathcal{M}_{\psi,2}^{'}(B+\frac{\alpha}{\sqrt{\delta}\Lambda}H; \frac{\alpha}{u\sqrt{\delta}\Lambda^2})] = -\frac{1}{2}\mathbb{E}[\left(\frac{1}{\Lambda} \mathcal{M}_{\psi,1}^{'}(B+\frac{\alpha}{\sqrt{\delta}\Lambda}H; \frac{\alpha}{u\sqrt{\delta}\Lambda^2})\right)^2]
\end{equation}
Using the above inequality, we can derive the following optimality conditions
\begin{equation}
\mathbb{E}[\frac{H}{\Lambda} \mathcal{M}_{\psi,1}^{'}(B+\frac{\alpha}{\sqrt{\delta}\Lambda}H; \frac{\alpha}{u\sqrt{\delta}\Lambda^2})] = u(1-\delta) 
\end{equation}

\begin{equation}
\mathbb{E}[\left(\frac{1}{\Lambda} \mathcal{M}_{\psi,1}^{'}(B+\frac{\alpha}{\sqrt{\delta}\Lambda}H; \frac{\alpha}{u\sqrt{\delta}\Lambda^2})\right)^2] = u^2(1-\delta) - \frac{\delta \sigma^2 u^2}{\alpha^2} 
\end{equation}

\textbf{Distributional convergence}

Next, to show distributional convergence, we first assume weak convergence of the solution of the AO, i.e.
\begin{equation}
    \hat{\bm \beta}^{AO}_{n,i} = \text{prox}_{\psi}(\beta_i^* + \frac{\alpha_n}{\sqrt{\delta}\Sigma_{i,i}^{1/2}}h_i; \frac{\alpha_n}{u_n\sqrt{\delta}\Sigma_{i,i}}) \xrightarrow[]{D} \text{prox}_{\psi}(B + \frac{\alpha_{\psi}}{\sqrt{\delta}\Lambda} H ; \frac{\alpha_{\psi}}{u_{\psi}\sqrt{\delta}\Lambda^2})
\end{equation}
and we want to show that 

\begin{equation}
    \hat{\bm \beta}^{PO}_{n,i} \xrightarrow[]{D} \text{prox}_{\psi}(B + \frac{\alpha_{\psi}}{\sqrt{\delta}\Lambda} H ; \frac{\alpha_{\psi}}{u_{\psi}\sqrt{\delta}\Lambda^2})
\end{equation}

Define 
\begin{equation}
    F_n(\bm{\hat \beta}_n,\bm \beta^{*},\Sigma) := \frac{1}{n}\sum_{i=1}^{n}f(\hat \beta_{n,i},\beta^{*}_i,\Sigma_{i,i})
\end{equation}
where $f$ is any bounded Lipschitz function. Also, define the limit
\begin{equation}
  \kappa:= \mathbb{E}[f(X(B,\Lambda,H),B,\Lambda^2)]  
\end{equation}

For any fixed $\epsilon>0$, define the set
\begin{equation}
 \mathcal{S}_{\epsilon} = \mathcal{S}_{\epsilon}(\bm \beta^{*},\bm \Sigma) = \{\bm w = \frac{1}{\sqrt{n}}\Sigma^{\frac{1}{2}}(\bm \beta - \bm \beta^*) \in \mathcal{W} \text{ s.t } |F_n(\bm \beta, \bm \beta^*, \Sigma) - \kappa|> 2\epsilon\}   
\end{equation}

Consider the following perturbed PO and AO problems 
\begin{equation}
\Phi_{\mathcal{S}_{\epsilon}}(\bm G) = \min_{\bm w \in \mathcal{S}_{\epsilon}} \psi(\bm \beta^* + \sqrt{n}\Sigma^{-1/2}\bm w) \quad \text{subject to} \quad \bm G \bm w = \sigma \bm z
\end{equation}
and 
\begin{equation}
\phi_{\mathcal{S}_{\epsilon}}(\bm g, \bm h) = \min_{\bm w \in \mathcal{S}_{\epsilon}} \psi(\bm \beta^* + \sqrt{n}\Sigma^{-1/2}\bm w) \quad \text{subject to} \quad \|\bm g \| \sqrt{\|\bm w\|^2 + \sigma^2} - \bm h^{T}\bm w - \sigma h \leq 0
\end{equation}
Using \cite{thrampoulidis2018precise} Theorem 6.1(iii), it is sufficient to show existence of constants $\bar \phi,\bar \phi_{\mathcal{S}_{\epsilon}}$ and $\eta >0$ satisfying
\begin{enumerate}
    \item $\bar \phi_{\mathcal{S}_{\epsilon}} \geq \bar \phi + 3\eta$
    \item $\phi_(\bm g, \bm h) \leq \bar \phi +\eta$ with probability approaching 1.
    \item $\phi_{\mathcal{S}_{\epsilon}}(\bm g, \bm h) \geq \bar \phi_{\mathcal{S}_{\epsilon}}- \eta$ with probability approaching 1.
\end{enumerate}
to prove that $\bm {\hat w}_n \notin \mathcal{S}_{\epsilon}$ wpa 1.

\textbf{Condition 2} Choose $\bar \phi = D(\alpha_*,u_*)$, we have shown that $\phi(\bm g,\bm h) \xrightarrow[]{P} \bar\phi$. So for any $\eta>0$, we have that 
\begin{equation}
    \bar \phi + \eta \geq \phi(\bm g,\bm h) \geq \bar \phi - \eta
\end{equation}

\textbf{Condition 3} Let $c(\bm w) := \|\bm g \| \sqrt{\|\bm w\|^2 + \sigma^2} - \bm h^{T}\bm w - \sigma h$, clearly $c$ is strictly convex in $\bm w$. At the optimum $\bm {\hat w}_{n}^{AO}$, we have that 
\begin{equation}
    -\lambda\nabla_{\bm w}c({\hat w}_{n}^{AO}) \in \partial_{\bm w} \Psi(\bm \beta^* + \sqrt{n}\Sigma^{-1/2}\bm {\hat w}_{n}^{AO}) 
\end{equation}
for $\lambda\geq0$ and also by feasibility, we have $c(\bm {\hat w}_{n}^{AO})= 0$; otherwise, we can always move along the negative gradient of the objective to reduce the objective value of the objective assuming that $\partial_{\bm w} \Psi(\bm \beta^* + \sqrt{n}\Sigma^{-1/2}\bm {\hat w}_{n}^{AO}) \text{\textbackslash} \{\bm 0\}$ is non-empty, which is true when $\psi$ has an unique minimizer. Next, we argue that in the new constrained formulation $\mathcal{S}_{\epsilon} \cap \{ \bm w \text{ s.t } c(\bm w) \leq 0\}$,  if $\|\bm w - \bm{\hat w}_{n}^{AO}\| \geq \Tilde{\epsilon}$, then the value of objective increases. By convexity of objective and optimality of $\bm {\hat w}_{n}^{AO}$, we have that

\begin{equation}
\Psi(\bm \beta^* + \sqrt{n}\Sigma^{-1/2}\bm w) \geq \phi(\bm g,\bm h) - \lambda\nabla_{\bm w}c({\hat w}_{n}^{AO})^{T}(\bm w - \bm {\hat w}_{n}^{AO})
\end{equation}

\textbf{Case 1} If $\nabla_{\bm w}c({\hat w}_{n}^{AO})^{T}(\bm w - \bm {\hat w}_{n}^{AO}) < 0$, then objective increases.

\textbf{Case 2} If $\nabla_{\bm w}c({\hat w}_{n}^{AO})^{T}(\bm w - \bm {\hat w}_{n}^{AO}) \geq 0$, then
    \begin{equation}
        c(\bm w) > c(\bm {\hat w}_{n}^{AO}) + \nabla_{\bm w}c({\hat w}_{n}^{AO})^{T}(\bm w - \bm {\hat w}_{n}^{AO})
    \end{equation}
and the inequality is strict due to strict convexity if the constraint function and therefore $c(\bm w) > 0$, therefore its not feasible. So the feasible option is Case 1 and there exists a constant $\lambda\nabla_{\bm w}c({\hat w}_{n}^{AO})^{T}(\bm w - \bm {\hat w}_{n}^{AO}) > q(\Tilde{\epsilon}) > 0$ such that
\begin{equation}
    \Psi(\bm \beta^* + \sqrt{n}\Sigma^{-1/2}\bm w) > \phi(\bm g,\bm h) + q(\Tilde{\epsilon})
\end{equation}
which implies that with probability approaching 1, we have
\begin{equation}
\Psi(\bm \beta^* + \sqrt{n}\Sigma^{-1/2}\bm w) > \bar \phi + q(\Tilde{\epsilon}) - \eta 
\end{equation}
Next, we argue that for small enough $\eta$, we have \textbf{Condition 1}, which is equivalent to showing that
\begin{equation}
    \bar \phi + q(\Tilde{\epsilon}) - \eta  - ( \bar \phi + \eta) \geq \eta
\end{equation}
Choosing $\eta = \frac{q(\Tilde{\epsilon})}{4}$ and $\bar \phi_{\mathcal{S}_{\epsilon}} = \bar \phi + q(\Tilde{\epsilon})$, satisfies the above inequality.

Next, we show $\|\bm w - {\hat w}_{n}^{AO}\| \geq \Tilde{\epsilon}$. By definition, we have
\begin{equation}
    |F_n(\bm \beta, \bm \beta^*, \Sigma) - \kappa|> 2\epsilon
\end{equation}
We have already shown that
\begin{equation}
    |F_n(\bm {\hat \beta}_{n}^{AO}, \bm \beta^*, \Sigma) - \kappa| \leq \epsilon
\end{equation}
By Lipschitzness and Cauchy-Schwarz, we get
\begin{equation}
    |F_n(\bm {\hat \beta}_{n}^{AO}, \bm \beta^*, \Sigma) - F_n(\bm \beta, \bm \beta^*, \Sigma)| \leq C \|\bm w - {\hat w}_{n}^{AO}\|
\end{equation}
$$
2\epsilon \leq |F_n(\bm \beta, \bm \beta^*, \Sigma) - \kappa| + |F_n(\bm {\hat \beta}_{n}^{AO}, \bm \beta^*, \Sigma) - F_n(\bm \beta, \bm \beta^*, \Sigma)| 
$$
$$
\leq \epsilon + C \|\bm w - {\hat w}_{n}^{AO}\|
$$
which implies  $\|\bm w - {\hat w}_{n}^{AO}\| \geq \Tilde{\epsilon}$

\end{proof}

\subsection{Special Cases}

In this section, we will look at the special cases of the minimum $\ell_1$, $\ell_2$, $\ell_3$, and $\ell_{\infty}$-norm interpolators. One can study the cases of minimum $\ell_p$-norms as a direct application of Theorem \ref{thm: PA}, since they are separable. But in order to get analytical expressions for the optimality conditions \eqref{eq: KKT}, one needs to restrict themselves to $\ell_p$-norms with $p=1,2 \text{ and } 3$. Going beyond $p=3$ involves solving for the roots of higher-order polynomials that don't have closed-form expressions in general, although they can still be solved numerically.

We instead provide an alternative approach that is equivalent to Theorem \ref{thm: PA}. To apply CGMT to \eqref{eq: sep implicit bias}, we first do a change of variable $\bm w:= \frac{\bm \Sigma^{1/2}(\bm \beta - \bm \beta^*)}{\sqrt{n}}$ and define $X:= \frac{1}{\sqrt{n}}\bm{G \Sigma}^{1/2}$, where $\bm G\in \mathbb{R}^{m\times n}$ is random matrix with iid $\mathcal{N}(0,1)$ entries. Now, \eqref{eq: sep implicit bias} can be rewritten in the form below, which we call the PO
\begin{equation}
 \label{eq: PO}
 \Phi(\bm G):= \min_{\bm w} \Psi(\bm \beta^* + \sqrt{n}\Sigma^{-1/2}\bm w) \quad \text{s.t} \quad \bm G \bm w =  \bm z.
\end{equation}
Using the constrained formulation of CGMT, previously used in \cite{li2020exploring, chang2021provable}, the corresponding AO becomes
\begin{align}
 \label{eq: AO}
 \phi(\bm g,\bm h):= &\min_{\bm w} \Psi(\bm \beta^* + \sqrt{n}\Sigma^{-1/2}\bm w)\\ 
 &\text{s.t} \quad \|\bm g \| \sqrt{\|\bm w\|^2 + \sigma^2} - \bm h^{T}\bm w - \sigma h \leq 0\nonumber
\end{align}
Note that the PO \eqref{eq: PO} is a function of the random matrix $\bm G$, whereas the AO \eqref{eq: AO} depends on independent random vectors $\bm g$ and $\bm h$. We now solve the constrained AO \eqref{eq: AO} directly using a geometric approach and verifying KKT conditions. This approach also gives a system of non-linear equations analogous to \eqref{eq: KKT}. Additionally, we also analyze the $\ell_{\infty}$-norm using this approach, which is not separable, so Theorem \ref{thm: PA} can't be applied directly to it. Previous work \cite{varma2023asymptotic} also characterizes the asymptotic distributions of SMD weights but only for linear regression when interpolating pure noise. This work strictly generalizes our previous work in terms of the structure of the underlying unknown signal and the feature covariance of the data.

\textbf{Notation:} We use the un-normalized versions of the $\ell_p$ norms as our potentials. So, if $\bm{x} = (x_1,x_2,\cdots,x_n)$ then for 
\begin{itemize}
\item $p<\infty$
$$
\psi(\bm{x}) = \sum_{i=1}^{n} |x_i|^p 
$$
\item $p=\infty$
$$
\psi(\bm{x}) = \max_{i} |x_i|
$$
\end{itemize}
Also, we now denote $\lambda_i = \Sigma_{i,i}^{1/2}$ and $\lambda_i^2 = \Sigma_{i,i}$.

Recall the constrained AO problem
\begin{align}
 \phi(\bm g,\bm h):= &\min_{\bm w} \Psi(\bm \beta^* + \sqrt{n}\bm \Sigma^{-1/2}\bm w)\\ 
 &\text{s.t} \quad \|\bm g \| \sqrt{\|\bm w\|^2 + \sigma^2} - \bm h^{T}\bm w - \sigma h \leq 0\nonumber
\end{align}
Let the constraint function $c_n(\bm w):=  \|\bm g \| \sqrt{\|\bm w\|^2 + \sigma^2} - \bm h^{T}\bm w - \sigma h$, if we normalize by $\sqrt{n}$ and take the asymptotic limit then $\frac{c_n(\bm w)}{\sqrt{n}} \xrightarrow[]{P} \sqrt{\delta}\sqrt{\|\bm w\|^2 + \sigma^2} - \bm h_e^{T}\bm w$. If we denote $\bm{w} = \alpha \bm{h}_e + \bm{h}^{\perp}$, where $\alpha>0$, $\bm{h}_e := \frac{\bm{h}}{\sqrt{n}}$ and $\bm{h}^{\perp}$ be the remaining orthogonal component and let $\|\bm{h}^{\perp}\|^2 = \gamma^2$, then $\|\bm{w}\|^2_{2} = \alpha^2 + \gamma^2$. Under this new re-parameterization, the asymptotic limit of the constraint becomes
\begin{equation}
 \label{eq: AO constraint}
 \alpha^2 \geq \Tilde{\delta}(\gamma^2 + \sigma^2)
\end{equation}
where $\Tilde{\delta}:= \frac{\delta}{1-\delta}$. \eqref{eq: AO constraint} captures the interpolation constraint of the PO for the AO problem. Therefore, the modified AO problem is
\begin{equation}
\label{eq: mod AO}
\min_{\bm{w}} \psi(\bm{\beta}_0 + \Sigma^{-\frac{1}{2}}\bm{w})\quad\text{s.t}\;\bm{w}\in C 
\end{equation}
where 
\begin{equation}
 C:= \{\bm{w} = \alpha \bm{h}_e + \bm{h}^{\perp} \;|  \alpha^2 \geq \Tilde{\delta}(\|\bm{h}^{\perp}\|^2 + \sigma^2) \}   
\end{equation}

and recall $\frac{1}{2}\nabla C(\bm{w}_{*}) = \alpha \bm{h_e}  - \Tilde{\delta}\bm{h}^{\perp}$. The proof technique involves assuming the structure of the solution and proving that the solution satisfies Karush–Kuhn–Tucker (KKT) conditions. Therefore, with a slight abuse of notation denote$\bm \beta^{*}_{e} = \frac{\bm \beta}{\sqrt{n}}$ and $\bm{w}_{*} = \alpha \bm{h}_e + \bm{h}^{\perp}$ with $\|\bm{h}^{\perp}\|^2 = \gamma^2$, then we have the following KKT conditions on $\bm{w}_{*}$:

\noindent \textbf{1) Stationarity:} If $\partial \psi(\bm{\beta}^{*} + \sqrt{n}\Sigma^{-\frac{1}{2}}\bm{w}_{*})$ is the subdifferential at $\bm{w}_{*}$, then $\exists k > 0$ such that 
\begin{equation}
\label{eq:kkt1}
k\nabla C(\bm{w}_{*}) \in \partial \psi(\bm{\beta}^{*} + \Sigma^{-\frac{1}{2}}\bm{w}_{*}), 
\end{equation}
where $\frac{1}{2}\nabla C(\bm{w}_{*}) = \alpha \bm{h_e}  - \Tilde{\delta}\bm{h}^{\perp}$.

\noindent \textbf{2) Primal feasibility:} We require $\bm{w^*} \!\!\in\!C(\bm{h},\sigma,\delta)$ which implies\looseness=-1
\begin{equation}
\label{eq:kkt2}
\alpha^2 = \Tilde{\delta}(\gamma^2 + \sigma^2) 
\end{equation}
Notice in \eqref{eq:kkt2}, we have replaced inequality with equality. This means that the optimal solution lies on the boundary of the constraint, which is true when the gradient of the objective is non-zero over the constraint set. Next, we define the following asymptotic limit of summations, which we use later in our analysis
\begin{itemize}
\item  
\begin{align*}
L^{1}_{a,b,c}(\alpha,\theta) &= \frac{1}{n}\sum_{i\;:|\alpha h_i + \delta \lambda_i \beta^{*}_{i}|\leq \frac{\theta}{\lambda_i}}(\lambda_i  \beta^{*}_{i})^a h_{i}^b\text{sign}(\alpha h_i + \delta \lambda_i \beta^{*}_{i})^c\\
&= \int_{0}^{\infty} \int_{-\infty}^{\infty} \int_{\frac{-\theta/\lambda - \delta \lambda w}{\alpha}}^{\frac{\theta/\lambda - \delta \lambda w}{\alpha}} (\lambda w)^a h^b\text{sign}(\alpha h + \delta \lambda w)^c P_H(h)P_B(w)P_{\Lambda}(\lambda) dh dw d\lambda
\end{align*}

\item  
\begin{align*}
U^{1}_{a,b,c}(\alpha,\theta) &= \frac{1}{n}\sum_{i\;:|\alpha h_i + \delta \lambda_i \beta^{*}_{i}|> \frac{\theta}{\lambda}}(\lambda_i \beta^{*}_{i})^a h_{i}^b\text{sign}(\alpha h_i + \delta \lambda_i \beta^{*}_{i})^c\\
&= \int_{0}^{\infty} \int_{-\infty}^{\infty} \int_{\frac{\theta/\lambda - \delta \lambda w}{\alpha}}^{\infty} (\lambda w)^a h^b\text{sign}(\alpha h + \delta \lambda w)^c P_H(h)P_B(w)P_{\Lambda}(\lambda) dh dw d\lambda\\
&+ \int_{0}^{\infty} \int_{-\infty}^{\infty} \int_{-\infty}^{\frac{-\theta/\lambda - \delta \lambda w}{\alpha}} (\lambda w)^a h^b\text{sign}(\alpha h + \delta \lambda w)^c P_H(h)P_B(w)P_{\Lambda}(\lambda) dh dw d\lambda
\end{align*}

\end{itemize}

\begin{itemize}
\item  
\begin{align*}
L^{\infty}_{a,b,c}(\alpha,\theta) &= \frac{1}{n}\sum_{i\;:|\alpha h_i + \delta \lambda_i \beta^{*}_{i}| < \lambda_i \theta}(\lambda_i \bar w_{i})^a h_{i}^b(\lambda_i\text{sign}(\alpha h_i + \delta \lambda_i \beta^{*}_{i}))^c\\
&= \int_{0}^{\infty} \int_{-\infty}^{\infty} \int_{\frac{-\lambda\theta - \delta \lambda w}{\alpha}}^{\frac{\lambda\theta - \delta \lambda w}{\alpha}} (\lambda w)^a h^b(\lambda\text{sign}(\alpha h + \delta \lambda w))^c P_H(h)P_B(w)P_{\Lambda}(\lambda) dh dw d\lambda
\end{align*}

\item  
\begin{align*}
U^{\infty}_{a,b,c}(\alpha,\theta) &= \frac{1}{n}\sum_{i\;:|\alpha h_i + \delta \lambda_i \beta^{*}_{i}|\geq \lambda_i\theta}(\lambda_i \bar w_{i})^a h_{i}^b(\lambda_i\text{sign}(\alpha h_i + \delta \lambda_i \beta^{*}_{i}))^c\\
&= \int_{0}^{\infty} \int_{-\infty}^{\infty} \int_{\frac{\lambda \theta - \delta \lambda w}{\alpha}}^{\infty} (\lambda w)^a h^b(\lambda\text{sign}(\alpha h + \delta \lambda w))^c p_H(h)p_W(w)p_{\Lambda}(\lambda) dh dw d\lambda\\
&+ \int_{0}^{\infty} \int_{-\infty}^{\infty} \int_{-\infty}^{\frac{-\lambda\theta - \delta \lambda w}{\alpha}} (\lambda w)^a h^b(\lambda\text{sign}(\alpha h + \delta \lambda w))^c P_H(h)P_B(w)P_{\Lambda}(\lambda) dh dw d\lambda
\end{align*}

\end{itemize}

\subsubsection{\texorpdfstring{$\ell_2$-norm}{}}

\begin{theorem}
\label{th:l2}
For a given $\delta$ and $\sigma^2$, let $\bm{h}_e \sim \mathcal{N}(0,\frac{1}{n}\bm{I}_n)$ then the solution of the AO in \eqref{eq: mod AO} for the $\ell_2$ potential is given as
\begin{equation}
w_{*i} = \frac{\alpha k\lambda_i^2(1+\Tilde{\delta})h_{ei} - \lambda_i\beta^{*}_{ei}}{k\Tilde{\delta}\lambda_i^2 +1} \quad \quad  \text{for $i \in [n]$} 
\end{equation}
and the weights of the $\bm{\hat \beta}^{AO}$ has the following empirical distribution
\begin{equation}
{\hat \beta}_{i} \sim \frac{k\Tilde{\delta}\Lambda^2}{1 + k\Tilde{\delta}\Lambda^2}B + \frac{\alpha k(1 + \Tilde{\delta})\Lambda}{1 + k\Tilde{\delta}\Lambda^2}H \quad \quad  \text{for $i \in [n]$}  
\end{equation}

where $k$ satisfies
\begin{equation}
 \mathbb{E}_{\lambda}[\frac{k\lambda -1}{1+k\Tilde{\delta}\lambda}] =0,   
\end{equation}
and $\alpha$ satisfies
\begin{equation}
\alpha^2 = \delta(\|\bm{w}_{*}\|^2 +\sigma^2).
\end{equation}
\end{theorem}

\begin{proof}

From \eqref{eq:kkt1}, we have $k\nabla C(\bm{w}_{*})_i = \nabla \psi(\bm{\beta}^{*} + \sqrt{n}\Sigma^{-\frac{1}{2}}\bm{w}_{*})_i$, which gives us the following constraint
$$
k(\alpha h_{ei} - \Tilde{\delta}h^{\perp}_i) = \frac{\beta^{*}_{ei}}{\lambda_i} + \frac{\alpha h_{ei}+h^{\perp}_i}{\lambda_i^2}
$$
which implies
$$
h^{\perp}_i = \frac{\alpha(k\lambda_i^2 - 1)h_{ei} - \lambda_i \beta^{*}_{ei}}{k\Tilde{\delta}\lambda_i^2 +1}.
$$
Plugging back $\bm{h}^{\perp}$, we obtain
\begin{enumerate}
\item  $w_{*i} = \frac{\alpha k\lambda_i^2(1+\Tilde{\delta})h_{ei} - \lambda_i\beta^{*}_{ei}}{k\Tilde{\delta}\lambda_i^2 +1}$
\item  $\hat \beta_i = \frac{k\Tilde{\delta}\lambda_i^2}{1 + k\Tilde{\delta}\lambda_i^2}\beta^{*}_{ei} + \frac{\alpha k(1 + \Tilde{\delta})\lambda_i}{1 + k\Tilde{\delta}\lambda_i^2}h_{i}$
\end{enumerate}

\textbf{Orthogonality of $\bm{h}^{\perp}$:}
$$
0 = \sum_{i}h_{ei}h^{\perp}_i = \alpha\sum_{i} h_{ei}^2\frac{k\lambda_i^2 - 1}{1+k\Tilde{\delta}\lambda_i^2} + \sum_{i} \frac{h_{ei}\lambda_i \beta^{*}_{ei}}{k\Tilde{\delta}\lambda_i^2 + 1}
$$
Due to independence of $\bm{h}$, $\bm{\beta}^{*}$ and the eigen spectrum $\lambda_i$ the above expression in the limit simplifies to
\begin{equation}
\mathbb{E}_{\Lambda}[\frac{k\Lambda^2 -1}{1+k\Tilde{\delta}\Lambda^2}] =0
\end{equation}
which can be solved by obtaining the value of k. Next, we look at the remaining Primal feasibility constraint to solve for $\alpha$.

\textbf{Primal Feasibility:} \eqref{eq:kkt2} imposes the following constraint
$$
\alpha^2 = \delta(\|\bm{w}_{*}\|^2 + \sigma^2) \leftrightarrow \alpha^2 = \Tilde{\delta}(\gamma^2 + \sigma^2)
$$
where $\gamma^2$ is computed as 
$$
\gamma^2 = \sum_{i} h^{\perp2}_i
$$
which simplifies to
\begin{equation}
\gamma^2 = \alpha^2 \mathbb{E}_{\Lambda}[\frac{(k\Lambda^2 -1)^2}{(k\Tilde{\delta}\Lambda^2 +1)^2}] + \|\bm{w}_{0}\|_2^2\cdot\mathbb{E}_{\lambda}[\frac{\Lambda^2}{(k\Tilde{\delta}\Lambda^2 +1)^2}]   
\end{equation}
Solving gives us $\alpha$.

\end{proof}

\subsubsection{\texorpdfstring{$\ell_1$-norm}{}}

\begin{theorem}
\label{th:l1}
For a given $\delta$ and $\sigma^2$, let $\bm{h}_e \sim \mathcal{N}(0,\frac{1}{n}\bm{I}_n)$ and if $\alpha,\theta_e \geq 0$ are unique solutions to (\ref{l1: orth},\ref{l1: cvx})), then the solution of the AO for the $\ell_1$ potential is given as
\begin{equation}
{w_{*i}}= \begin{cases} -\lambda_i\beta^{*}_{ei}, & |\alpha h_{ei} + \delta \lambda_i \beta^{*}_{ei}|\leq \frac{\theta_e}{\lambda_i} \\\frac{\alpha h_{ei}}{\delta} - \frac{\theta_e}{\delta}\text{sign}(\alpha h_{ei} + \delta \lambda_i  \beta^{*}_{ei}), & |\alpha h_{ei} + \delta \lambda_i  \beta^{*}_{ei}|> \frac{\theta_e}{\lambda_i}
\end{cases}
\end{equation}
and the weights of the $\bm{\hat \beta}^{AO}$ for the $\ell_1$ potential has the following empirical distribution 
\begin{equation}
\label{eq: l1 dist}
{\hat \beta}^{AO}_{i} \sim \begin{cases} 0, & |\alpha H + \delta \Lambda B|\leq \frac{\theta}{\Lambda} \\B + \frac{\alpha H}{\delta\Lambda} - \frac{\theta\text{sign}(\alpha H + \delta \Lambda B)}{\delta\Lambda^2}, & |\alpha H + \delta \Lambda  B|> \frac{\theta}{\Lambda}
\end{cases}
\end{equation}
\end{theorem}

\begin{proof}

For $\ell_1$-norm $\partial \psi(x)$ is defined as
$$
\partial\psi(x)_i= \begin{cases} 1, & x_i>0 \\ -1, & x_i<0 \\ [-1,1], & x_i =0 \end{cases}
$$
We consider the following cases
\begin{itemize}
\item \textbf{case 1}: If  $i \in \{j\;:|\alpha h_j + \delta \lambda_j \beta^{*}_i|\leq \frac{\theta}{\lambda_i}\}$ then $h_{i}^{\perp} = -\lambda_i\beta^{*}_{ei} - \alpha h_{ei}$ , plugging back we have 
\begin{enumerate}
\item  $w_{*i} = -\lambda_i\beta^{*}_{ei}$
\item  $\hat \beta_i = 0$
\item $k\nabla C_i = \frac{2k}{1-\delta}(\alpha h_i + \delta \lambda_i \beta^{*}_{i})$
\end{enumerate}
To satisfy \eqref{eq:kkt1}, we set $\frac{2k\theta_e}{1-\delta} = 1$ which implies
$$
k = \frac{1-\delta}{2\theta_e} 
$$
where $\theta_e := \frac{\theta}{\sqrt{n}}$ is the normalised version of $\theta$.

\item \textbf{case 2}: If  $i \in \{j\;:|\alpha h_j + \delta \lambda_j \beta^{*}_i| > \frac{\theta}{\lambda_i}\}$ then $h_{i}^{\perp} = \frac{\alpha h_{ei}(1-\delta)}{\delta} - \frac{\theta_e}{\delta\lambda_i}\text{sign}(\alpha h_i + \delta \lambda_i \beta^{*}_{i})$ , plugging back we have 
\begin{enumerate}
\item  $w_{*i} = \frac{\alpha h_{ei}}{\delta} - \frac{\theta_e}{\delta \lambda_i}\text{sign}(\alpha h_i + \delta \lambda_i \beta^{*}_{i})$
\item  $\hat \beta_i = \beta^{*}_{i} + \frac{\alpha h_{i}}{\delta\lambda_i} - \frac{\theta\text{sign}(\alpha h_i + \delta \lambda_i \beta^{*}_{i})}{\delta\lambda_i^2}$
\item $k\nabla C_i = \frac{2k\theta_e}{(1-\delta)\lambda_i}\text{sign}(\alpha h_i + \delta \lambda_i \beta^{*}_{i})$
\end{enumerate}
Since we set $\frac{2k\theta_e}{1-\delta} = 1$, we have $|k\nabla C_i|=\frac{1}{\lambda_i}$ as desired.
\end{itemize}

Next, we verify the existence of constants $\theta, \alpha \geq 0$ such that KKT conditions hold for the above-claimed solution. Notice, that our solution structure of $\bm{h}^{\perp}$ satisfies stationarity \eqref{eq:kkt1} by its definition, but we now have an additional constraint to show $\bm{h}^{\perp}$ is orthogonal to $\bm{h}$.

\textbf{Orthogonality of $\bm{h}^{\perp}$:}
$$
0 = \sum_{i}h_{ei}h^{\perp}_i = \sum_{i\;:|\alpha h_i + \delta \lambda_i \beta^{*}_{i}|\leq \frac{\theta}{\lambda_i}} h_{ei}h_i^{\perp} + \sum_{i\;:|\alpha h_i + \delta \lambda_i \beta^{*}_{i}| > \frac{\theta}{\lambda_i}} h_{ei}h_i^{\perp}
$$
This simplifies to 
\begin{equation}
\label{l1: orth}
    \alpha U^1_{0,2,0}(\alpha,\theta) = \alpha\delta + \delta L^1_{1,1,0}(\alpha,\theta) + \theta U^1_{0,1,1}(\alpha,\theta)
\end{equation}
\textbf{Primal Feasibility:} \eqref{eq:kkt2} imposes the following constraint
$$
\alpha^2 = \delta(\|\bm{w}_{*}\|^2 + \sigma^2),
$$
where $\|\bm{w}_{*}\|^2$ is computed as 
$$
\|\bm{w}_{*}\|^2 = \sum_{i\;:|\alpha h_i + \delta \lambda_i \beta^{*}_{i}|\leq \frac{\theta}{\lambda_i}} w_{*i}^2 + \sum_{i\;:|\alpha h_i + \delta \lambda_i \beta^{*}_{i}| > \frac{\theta}{\lambda_i}} w_{*i}^2
$$
which simplifies to
\begin{equation}
\label{l1: cvx}
\|\bm{w}_{*}\|^2 = L^1_{2,0,0}(\alpha,\theta) + \frac{\alpha^2}{\delta^2}U^1_{0,2,0} + \frac{\theta^2}{\delta^2}U^1_{0,0,2} - \frac{2\alpha \theta}{\delta^2}U^1_{0,1,1}(\alpha,\theta)    
\end{equation}
If one can numerically solve for $\alpha,\theta$, we have verified the correctness of the solution, finishing the proof.

\end{proof}

\subsubsection{\texorpdfstring{$\ell_3$-norm}{}}

\begin{theorem}
\label{th:l3}
For a given $\delta$ and $\sigma^2$, let $\bm{h}_e \sim \mathcal{N}(0,\frac{1}{n}\bm{I}_n)$ and $\alpha,k \geq 0$ be the unique solutions of the non-linear equations (\ref{l3: orth},\ref{l3: cvx}), the solution of the AO for the $\ell_3$ potential is given as
\begin{equation}
{w_{*i}}= \begin{cases} \sqrt{k_e^2\lambda_i^6\Tilde{\delta}^2 + \frac{2k_e\lambda_i^3}{1-\delta}(\alpha h_{ei} + \delta \lambda_i \beta^{*}_{ei})} - \lambda_i \beta^{*}_{ei} - k_e\lambda_i^3\Tilde{\delta}, & \alpha h_{ei} + \delta \lambda_i \beta^{*}_{ei} \geq 0 \\-\sqrt{k_e^2\lambda_i^6\Tilde{\delta}^2 - \frac{2k_e\lambda_i^3}{1-\delta}(\alpha h_{ei} + \delta \lambda_i \beta^{*}_{ei})} - \lambda_i \beta^{*}_{ei} + k_e\lambda_i^3\Tilde{\delta}, & \alpha h_{ei} + \delta \lambda_{i}  \beta^{*}_{ei} < 0
\end{cases}
\end{equation}
and the weights of the $\bm{\hat \beta}^{AO}$ for the $\ell_3$ potential has the following empirical distribution 
\begin{equation}
{\hat \beta}^{AO}_{i} \sim \begin{cases} \sqrt{k^2\Lambda^4\Tilde{\delta}^2 + \frac{2k\Lambda}{1-\delta}(\alpha H + \delta \Lambda B)} - k\Lambda^2\Tilde{\delta}, & \alpha H + \delta \Lambda  B \geq 0 \\-\sqrt{k^2\Lambda^4\Tilde{\delta}^2 - \frac{2k \Lambda}{1-\delta}(\alpha H + \delta \Lambda B)} + k\Lambda^2\Tilde{\delta}, & \alpha B + \delta \Lambda  B < 0
\end{cases}
\end{equation}
\end{theorem}

\begin{proof}
    
At solution point there exists $k_e = \frac{k}{\sqrt{n}}$ such that $3k \nabla C(\bm{w}^*) = \nabla\psi(\bm{\beta}^{*} + \sqrt{n}\bm\Sigma^{-\frac{1}{2}}\bm{w}_{*}) $ which enforces
$$
2k(\alpha h_{ei} - \Tilde{\delta}h^{\perp}_i)=(\beta^{*}_{ei} + \frac{\alpha h_{ei} + h^{\perp}_i}{\lambda_i})^2 \frac{\text{sign}(\hat w_i)}{\lambda_i}.
$$
Notice that we have a quadratic equation for $h^{\perp}_i$. If we assume $\text{sign}(\alpha h_i + \delta \lambda_i \beta^{*}_{i}) = \text{sign}(\hat \beta_i)$ and  solve for the roots, we arrive at the following two cases

\begin{itemize}
\item \textbf{case 1}: If $i \in \{j\;:\alpha h_j + \delta \lambda_j \beta^{*}_i \geq 0\}$ then 
$$
h_{i}^{\perp} = \sqrt{k_e^2\lambda_i^6\Tilde{\delta}^2 + \frac{2k_e\lambda_i^3}{1-\delta}(\alpha h_{ei} + \delta \lambda_i \beta^{*}_{ei})} - \alpha h_{ei} - \lambda_i \beta^{*}_{ei} - k_e\lambda_i^3\Tilde{\delta},
$$
Plugging back, we have 
\begin{enumerate}
\item  $w_{*i} = \sqrt{k_e^2\lambda_i^6\Tilde{\delta}^2 + \frac{2k_e\lambda_i^3}{1-\delta}(\alpha h_{ei} + \delta \lambda_i \beta^{*}_{ei})} - \lambda_i \beta^{*}_{ei} - k_e\lambda_i^3\Tilde{\delta}$
\item  $\hat \beta_i = \sqrt{k^2\lambda_i^4\Tilde{\delta}^2 + \frac{2k\lambda_i}{1-\delta}(\alpha h_{i} + \delta \lambda_i \beta^{*}_{i})} - k\lambda_i^2\Tilde{\delta}$ 
\end{enumerate}

\item \textbf{case 2}: If $i \in \{j\;:\alpha h_j + \delta \lambda_j \beta^{*}_i < 0\}$ then 
$$
h_{i}^{\perp} = -\sqrt{k_e^2\lambda_i^6\Tilde{\delta}^2 - \frac{2k_e\lambda_i^3}{1-\delta}(\alpha h_{ei} + \delta \lambda_i \beta^{*}_{ei})} - \alpha h_{ei} - \lambda_i \beta^{*}_{ei} + k_e\lambda_i^3\Tilde{\delta},
$$
Plugging back, we have 
\begin{enumerate}
\item  $w_{*i} = -\sqrt{k_e^2\lambda_i^6\Tilde{\delta}^2 - \frac{2k_e\lambda_i^3}{1-\delta}(\alpha h_{ei} + \delta \lambda_i \beta^{*}_{ei})} - \lambda_i \beta^{*}_{ei} + k_e\lambda_i^3\Tilde{\delta}$
\item  $\hat \beta_i = -\sqrt{k^2\lambda_i^4\Tilde{\delta}^2 - \frac{2k \lambda_i}{1-\delta}(\alpha h_{i} + \delta \lambda_i \beta^{*}_{i})} + k\lambda_i^2\Tilde{\delta}$ 
\end{enumerate}

The above solution structure indeed satisfies $\text{sign}(\alpha h_i + \delta \lambda_i \beta^{*}_{i}) = \text{sign}(\hat \beta_i)$. Next, we verify the orthogonality of $\bm{h}^{\perp}$ and the remaining KKT condition to compute $\alpha$ and $k$ and finish the proof.

\textbf{Orthogonality of $\bm{h}^{\perp}$:}
Due to symmetry around zero
$$
\sum_{i=1}^{n} h_{ei}h_i^{\perp} = 0 = \sum_{\alpha h_j + \delta \lambda_j \beta^{*}_i >0} h_{ei}h_i^{\perp},
$$
which implies
\begin{equation*}
\sum_{\alpha h_j + \delta \lambda_j \beta^{*}_i >0}  h_{ei}\sqrt{k_e^2\lambda_i^6\Tilde{\delta}^2 + \frac{2k_e\lambda_i^3}{1-\delta}(\alpha h_{ei} + \delta \lambda_i \beta^{*}_{ei})}   
- h_{ei}(\alpha h_{ei} + \lambda_i \beta^{*}_{ei} + k_e\lambda_i^3\Tilde{\delta}) =0.
\end{equation*}
In the asymptotic limit, we have the following integral condition:
\begin{multline}
\label{l3: orth}
0 = \int_{0}^{\infty}\left(\int_{-\infty}^{\infty}\left(\int_{-\frac{\delta\lambda w}{\alpha}}^{\infty} h\lambda\sqrt{k^2\lambda^4\Tilde{\delta}^2 + \frac{2k \lambda}{1-\delta}(\alpha h + \delta \lambda w)} p_{H}(h)dh \right)p_{W}(w)dw\right)p_{\Lambda}(\lambda)d\lambda \\ 
- \int_{0}^{\infty}\left(\int_{-\infty}^{\infty}\left(\int_{-\frac{\delta\lambda w}{\alpha}}^{\infty} h(\alpha h + \lambda w + k\lambda^3\Tilde{\delta}) p_{H}(h)dh \right)p_{W}(w)dw\right)p_{\Lambda}(\lambda)d\lambda
\end{multline}

\textbf{Primal Feasibility:} \eqref{eq:kkt2} imposes the following constraint
$$
\alpha^2 = \delta(\|\bm{w}_{*}\|^2 + \sigma^2) 
$$
where $\|\bm{w}_{*}\|^2$ is computed as 
$$
\|\bm{w}_{*}\|^2 = \sum_{i} w_{*i}^2
$$
which in the limit is
\begin{equation}
\label{l3: cvx}
\|\bm{w}_{*}\|^2 =    2\int_{0}^{\infty}\int_{-\infty}^{\infty}\int_{-\frac{\delta\lambda w}{\alpha}}^{\infty} \left(\lambda\sqrt{k^2\lambda^4\Tilde{\delta}^2 + \frac{2k \lambda}{1-\delta}(\alpha h + \delta \lambda w)} - \lambda w - k\lambda^3\Tilde{\delta}\right)^2 p(h,w,\lambda)d(h,w,\lambda)
\end{equation}

\end{itemize}

\end{proof}

\subsubsection{\texorpdfstring{$\ell_{\infty}$-norm}{}}

\begin{theorem}
\label{th:linfinity}
For a given $\delta$ and $\sigma^2$, let $\bm{h}_e \sim \mathcal{N}(0,\frac{1}{n}\bm{I}_n)$ and $\alpha,\theta \geq 0$ be the unique minimizers of (\ref{linf: orth},\ref{linf: cvx}) then the solution of the AO for the $\ell_{\infty}$ potential is given as
\begin{equation}
\bm{w_{*i}}= \begin{cases} \frac{\alpha h_{ei}}{\delta}, & |\alpha h_{ei} + \delta \lambda_i \beta^{*}_{ei}|< \lambda_i \theta_e \\\frac{\lambda_i\theta_e}{\delta}\text{sign}(\alpha h_{ei} + \delta \lambda_i  \beta^*_{ei}) - \lambda_i\beta^{*}_{ei}, & |\alpha h_{ei} + \delta \lambda_i \beta^{*}_{ei}|\geq \lambda_i \theta_e
\end{cases}
\end{equation}
where $\theta_{e}=\frac{\theta}{\sqrt{n}}$ and the weights of the $\bm{\hat \beta}^{AO}$ for the $\ell_{\infty}$ potential has the following empirical distribution 
\begin{equation}
\label{eq: linf dist}
{\hat \beta}^{AO}_{i} \sim \begin{cases} B + \frac{\alpha H}{\delta\Lambda}, & |\alpha H + \delta \Lambda  B|< \Lambda \theta \\ \frac{\theta}{\delta}\text{sign}(\alpha H + \delta \Lambda B), & |\alpha H + \delta \Lambda B| \geq \Lambda \theta
\end{cases}
\end{equation}
\end{theorem}

\begin{proof}

For $\ell_{\infty}$-norm $\partial \psi(x)$ is defined as
$$
\partial\psi(x)_i= \begin{cases} 0, & |x_i|<\max_{j}|x_j| \\ [-1,0], & -x_i=\max_{j}|x_j| \\ [0,1], & x_i=\max_{j}|x_j| \end{cases}
$$
and additionally if $\nu \in \partial\psi(x)$ then $\|\nu\|_1 = 1$. Next, we consider the following cases
\begin{itemize}
\item \textbf{case 1}: If $i \in \{j\;:|\alpha h_j + \delta \lambda_j \beta^{*}_i|< \lambda_i\theta\}$ then $h_{i}^{\perp} = \frac{\alpha h_{ei}(1-\delta)}{\delta}$, plugging back we have 
\begin{enumerate}
\item  $w_{*i} = \frac{\alpha h_{ei}}{\delta}$
\item  $\hat \beta_i = \beta^{*}_{i} + \frac{\alpha h_{i}}{\delta\lambda_i}$ 
\item $k\nabla C_i = 0$
\end{enumerate}
Notice that $|\hat \beta_i| = |\beta^{*}_{i} + \frac{\alpha h_{i}}{\delta\lambda_i}| = \frac{1}{\delta\lambda_i}|\alpha h_{i} + \delta\lambda_i\beta^{*}_{i}| < \frac{\theta}{\delta}$. We'll see next that the $\ell_{\infty}$ potential indeed uniformly bounds the absolute values of the solution weights by $\frac{\theta}{\delta}$.

\item \textbf{case 2}: If  $i \in \{j\;:|\alpha h_j + \delta \lambda_j \beta^{*}_i| \geq \lambda_i\theta\}$ then $h_{i}^{\perp} = \frac{\lambda_i\theta_e}{\delta}\text{sign}(\alpha h_i + \delta \lambda_i \beta^{*}_{i}) - \lambda_i\beta^{*}_{ei} -\alpha h_{ei}$  , plugging back we have 
\begin{enumerate}
\item  $w_{*i} =\frac{\lambda_i\theta_e}{\delta}\text{sign}(\alpha h_i + \delta \lambda_i \beta^{*}_{i}) - \lambda_i\beta^{*}_{ei}$
\item  $\hat \beta_i =\frac{\theta}{\delta}\text{sign}(\alpha h_i + \delta \lambda_i \beta^{*}_{i})$
\item $k\nabla C_i = \frac{2k}{1-\delta} \left( \alpha h_{ei} + \delta \lambda_i \bar \beta^{*}_{ei} - \lambda_i\theta_e\text{sign}(\alpha h_i + \delta \lambda_i \beta^{*}_{i})\right)$
\end{enumerate}
Since $\|k\nabla C\|_1 = 1$, we set $k = \frac{1}{\|\nabla C\|_1}$.
\end{itemize}

Next, we verify the existence of constants $\theta, \alpha \geq 0$ such that KKT conditions hold for the above-claimed solution. Notice that our solution structure of $\bm{h}^{\perp}$ satisfies stationarity \eqref{eq:kkt1} by its definition, but we now have an additional constraint to show $\bm{h}^{\perp}$ is orthogonal to $\bm{h}$.

\textbf{Orthogonality of $\bm{h}^{\perp}$:}
$$
0 = \sum_{i}h_{ei}h^{\perp}_i = \sum_{i\;:|\alpha h_i + \delta \lambda_i \beta^{*}_{i}|< \lambda_i\theta} h_{ei}h_i^{\perp} + \sum_{i\;:|\alpha h_i + \delta \lambda_i \beta^{*}_{i}| \geq \lambda_i\theta} h_{ei}h_i^{\perp}
$$
This simplifies to 
\begin{equation}
\label{linf: orth}
    \alpha L^{\infty}_{0,2,0}(\alpha,\theta) + \theta U^{\infty}_{0,1,1}(\alpha,\theta)= \alpha\delta + \delta U^{\infty}_{1,1,0}(\alpha,\theta) 
\end{equation}
\textbf{Primal Feasibility:} \eqref{eq:kkt2} imposes the following constraint
$$
\alpha^2 = \delta(\|\bm{w}_{*}\|^2 + \sigma^2),
$$
where $\|\bm{w}_{*}\|^2$ is computed as 
$$
\|\bm{w}_{*}\|^2 = \sum_{i\;:|\alpha h_i + \delta \lambda_i \beta^{*}_{i}|< \lambda_i\theta} w_{*i}^2 + \sum_{i\;:|\alpha h_i + \delta \lambda_i \beta^{*}_{i}| \geq \lambda_i \theta} w_{*i}^2
$$
which simplifies to
\begin{equation}
\label{linf: cvx}
\|\bm{w}_{*}\|^2 = \frac{\alpha^2}{\delta^2}L^{\infty}_{0,2,0}(\alpha,\theta) + \frac{\theta^2}{\delta^2}U^{\infty}_{0,0,2} + U^{\infty}_{2,0,0} - \frac{2\theta}{\delta}U^{\infty}_{1,0,1}(\alpha,\theta)    
\end{equation}
If one can numerically solve for $\alpha,\theta$, we have verified the correctness of the solution, finishing the proof.

\end{proof}

Theorems \ref{th:l1}, \ref{th:l2}, \ref{th:l3}, and \ref{th:linfinity} provide exact characterizations of the solutions of the AO, and the generalization error can be found by computing the Euclidean norm square. Next, we numerically verify our results by comparing the theoretical and experimental distribution of the weights for these different norms. Figure \ref{fig: hist potentials} provides this comparison of the theoretical results derived with the experimental weights obtained by running SMD for $n=3000$, $\delta=0.3$, $\sigma=0.3$ and Gaussian $\bm{\beta}^*$. $\ell_{1.05}$ and $\ell_{20}$ SMD are used as a proxy for $\ell_{1}$ and $\ell_{\infty}$ respectively since they are not practically tractable. Finally, Figure \ref{fig: emp vs the ER} shows a comparison of the excess risk with different convex potentials and underlying signal structures, and shows the concentration of its non-asymptotic values to the theoretical limit when $d=2000$.

\begin{figure}[h]
\centering
\begin{tabular}{cccc}
    \includegraphics[width=.24\textwidth]{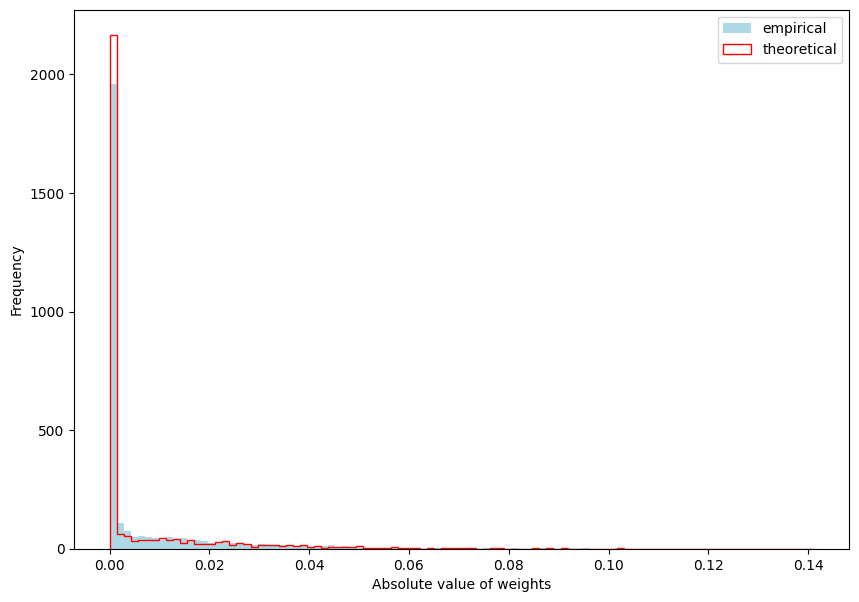} & \includegraphics[width=.24\textwidth]{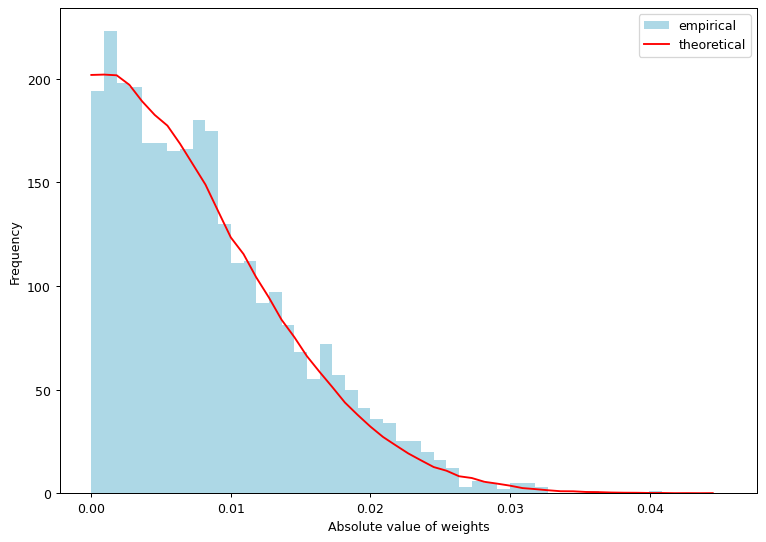}& \includegraphics[width=.24\textwidth]{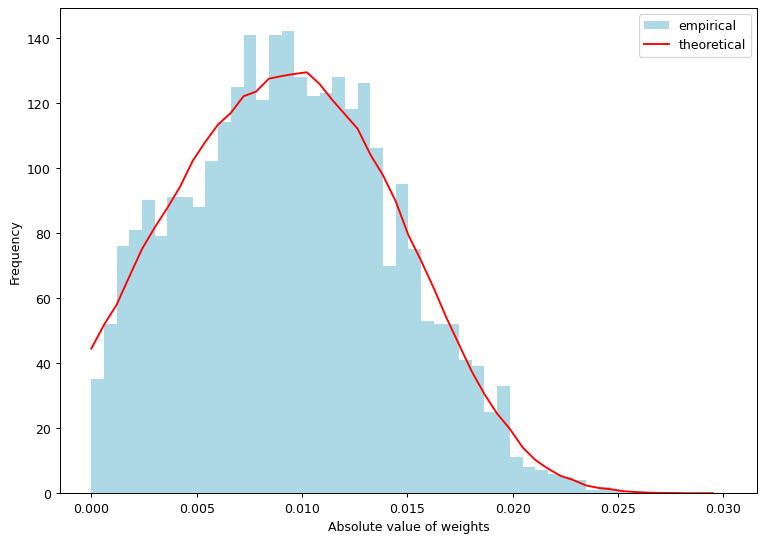}& \includegraphics[width=.24\textwidth]{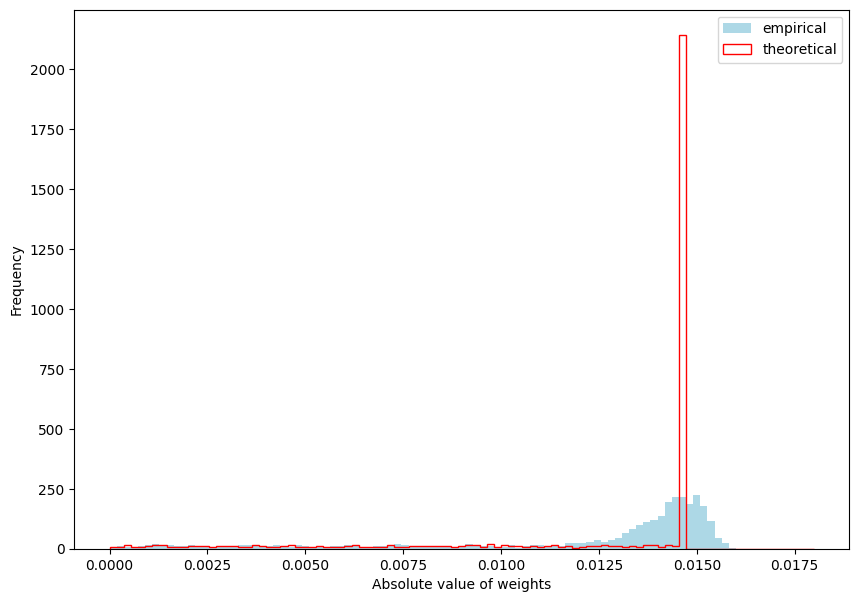} \\
     (a)&(b)&(c)&(d) 
\end{tabular}

    \caption{ Empirical and theoretical distribution of the solution weights for the SMD for $n=3000$, $\delta=0.3$, $\sigma=0.3$ and Gaussian $\bm{\beta}^*$ with (a) $\ell_{1}$, (b) $\ell_2$, (c) $\ell_{3}$-norm and (d) $\ell_{\infty}$ potentials.}
    \label{fig: hist potentials}
\end{figure}

\begin{figure}[h]
\centering
\begin{tabular}{ccc}
     \includegraphics[width=.29\textwidth]{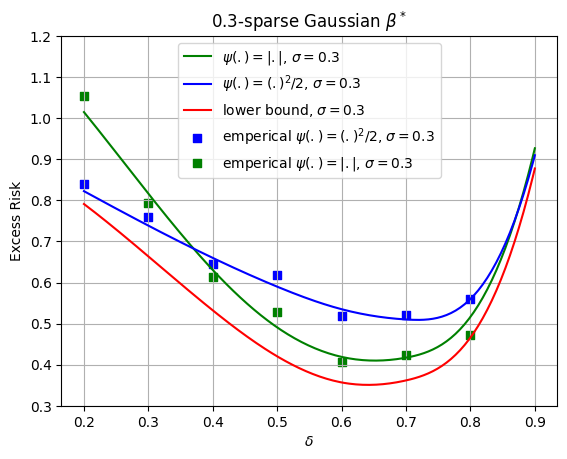}& \includegraphics[width=.29\textwidth]{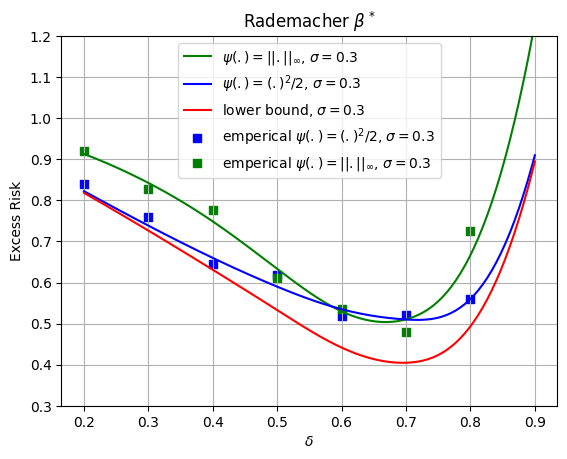}&
     \includegraphics[width=.29\textwidth]{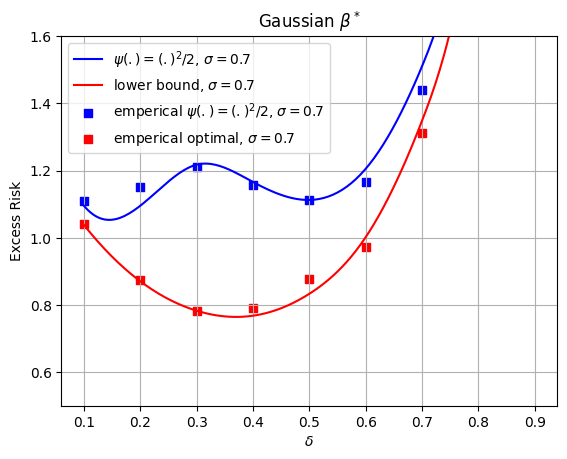}\\
     (a)&(b)&(c) 
\end{tabular}

    \caption{ Empirical and theoretical Excess risk with $d=3000$ (a) when $P_{B}$ has a 0.3-sparse Gaussian density and $\Lambda = 1$ a.s (b) when $P_{B}$ has a Rademacher density and $\Lambda = 1$ a.s (c) when $P_{B}$ has a Gaussian density and $\Lambda^2$ takes the value $4$ with probability $0.3$ and $0.1$ with probability $0.7$.}
    \label{fig: emp vs the ER}
\end{figure}

\section{Fundamental limits}
\label{apx: fund lim}

\subsection{Proof of Theorem \ref{thm: lower bound}}

\begin{theorem}(Lower bound on $\alpha_{\psi}^2$)
Let Assumptions \ref{ass: HDA},\ref{ass: Gaussian feat and noise}, and \ref{ass: true dist} hold. Define the random variable $V_{\alpha} = B + \frac{\alpha}{\sqrt{\delta}\Lambda}H$ where $H\sim \mathcal{N}(0,1)$ and $B,\Lambda$ as defined in assumption \ref{ass: Gaussian feat and noise} and \ref{ass: true dist}. Let $\alpha_{*}$ be the unique solution of the following non-linear equation
\begin{equation}
\label{eq: lower bound apx}
 \alpha^2= \frac{\delta\sigma^2}{1-\delta}+\frac{\delta(1-\delta)}{\mathcal{I}_{\Lambda}(V_{\alpha})}
\end{equation}
where $\mathcal{I}_{\Lambda}(V_{\alpha})$ is the weighted fisher information of $V_{\alpha}$ defined as
\begin{equation}
\label{eq: weighted fisher info apx}
    \mathcal{I}_{\Lambda}(V_{\alpha}) := \mathbb{E}\left[ \left(\frac{\xi_{V_{\alpha}}({V_{\alpha}|\Lambda)}}{\Lambda}\right)^2\right]
\end{equation}
where $\xi_{V_{\alpha}}(v|\Lambda):= \frac{p^{'}_{V_{\alpha}(v|\Lambda)}}{p_{V_{\alpha}(v|\Lambda)}}$ is the conditional score function of $V_{\alpha}$. Then for every $\Psi \in \mathcal{C}_{\psi}$, with $\alpha^2_{\psi}$ as the asymptotic limit of the generalization error as in \eqref{eq: asym ge}, we have $\alpha^2_{\psi} \geq \alpha_{*}^2$.
\end{theorem}

\begin{proof}
Recall the system of non-linear equations and let $(\alpha_{\Psi},u_{\Psi})$ be a solution
\begin{subequations}
 \label{eq: KKT apx}
 \begin{align}
\mathbb{E}\left[\frac{H}{\Lambda} \mathcal{M}_{\psi,1}^{'}(B+\frac{\alpha}{\sqrt{\delta}\Lambda}H; \frac{\alpha}{u\sqrt{\delta}\Lambda^2})\right] &= u(1-\delta) \\
\mathbb{E}\left[\left(\frac{1}{\Lambda} \mathcal{M}_{\psi,1}^{'}(B+\frac{\alpha}{\sqrt{\delta}\Lambda}H; \frac{\alpha}{u\sqrt{\delta}\Lambda^2})\right)^2\right] &= u^2(1-\delta) - \frac{\delta \sigma^2 u^2}{\alpha^2}
 \end{align}
\end{subequations}
and let $\alpha_{*}$ be the solution of \eqref{eq: lower bound apx}. First, we argue that the solution $\alpha_{*}$ always exists when $\sigma>0$. Consider the function 
$$
h(\alpha) = \frac{\delta\sigma^2}{\alpha^2(1-\delta)} + \frac{\delta(1-\delta)}{\alpha^2\mathcal{I}_{\Lambda}(V_{\alpha})}
$$
Therefore, at $\alpha_{*}$, we have $h(\alpha_{*})=1$. Note that $h(\alpha)$ is continuous on $\mathbb{R}_{>0}$ and when $\sigma>0$, we have that $\lim_{\alpha \xrightarrow[]{} 0^{+}} h(\alpha) = \infty$ and  $\lim_{\alpha \xrightarrow[]{} \infty} h(\alpha) = 1-\delta <1$ using the fact that $\lim_{\alpha \xrightarrow[]{} \infty} \alpha^2\mathcal{I}_{\Lambda}(V_{\alpha}) = \delta$ using properties from Proposition \ref{prop: FI limits}. Therefore, by the mean value theorem, we can argue that the existence of $\alpha_{*}$. To show the uniqueness of $\alpha_{*}$, we need to show that $h(\alpha)$ is monotonically decreasing. Using properties of Fisher information $\alpha^2\mathcal{I}_{\Lambda}(V_{\alpha}) = \mathcal{I}_{\Lambda}(\frac{V_{\alpha}}{\alpha})$ and one can verify that $\mathcal{I}_{\Lambda}(\frac{V_{\alpha}}{\alpha})$ is monotonically increasing using Proposition \ref{prop: FI} (c), as a consequence $h(\alpha)$ is monotonically decreasing. Now that we have established the existence and uniqueness of $\alpha_{*}$, we next show that $\alpha_{*}$ is a lower on $\alpha_{\Psi}$.

Consider the following integral
\begin{multline}
 \frac{\sqrt{\delta}}{\alpha}\mathbb{E}[\frac{\alpha H}{\sqrt{\delta}\Lambda} \mathcal{M}_{\psi,1}^{'}(B+\frac{\alpha}{\sqrt{\delta}\Lambda}H; \frac{\alpha}{u\sqrt{\delta}\Lambda^2})] =  \\
 \frac{\sqrt{\delta}}{\alpha}\iiint g \mathcal{M}_{\psi,1}^{'}(b + g; \frac{\alpha}{u\sqrt{\delta}\lambda^2})p_{B}(b)p_{G}(g|\lambda)p_{\Lambda}(\lambda)\text{d}b \text{d}g \text{d}\lambda 
\end{multline}
where $G$ is a conditionally Gaussian random variable $p_{G}(.|\lambda) \sim \mathcal{N}(0,\frac{\alpha^2}{\delta\lambda^2})$. Next, we use the following using the property of Gaussian density
\begin{equation}
    p_{G}^{'}(g|\lambda) = - g\frac{\delta\lambda^2}{\alpha^2}p_{G}(g|\lambda)
\end{equation}
Plugging back, we get
\begin{equation}
 -\frac{\alpha}{\sqrt{\delta}} \iiint \frac{1}{\lambda^2} \mathcal{M}_{\psi,1}^{'}(b + g; \frac{\alpha}{u\sqrt{\delta}\lambda^2})p^{'}_{G}(g|\lambda)p_{B}(b)p_{\Lambda}(\lambda)\text{d}b \text{d}g \text{d}\lambda    
\end{equation}
If consider, the following change of variable $v = b + g$, then $\text{d}v = \text{d}b$ and
\begin{equation}
 -\frac{\alpha}{\sqrt{\delta}} \iiint \frac{1}{\lambda^2} \mathcal{M}_{\psi,1}^{'}(v; \frac{\alpha}{u\sqrt{\delta}\lambda^2})p^{'}_{G}(g|\lambda)p_{B}(v-g)p_{\Lambda}(\lambda)\text{d}v \text{d}g \text{d}\lambda    
\end{equation}
Integrating $g$ out using Gaussian Integration of parts, one can verify that
\begin{equation}
    \int_{-\infty}^{\infty}p^{'}_{G}(g|\lambda)p_{B}(v-g) = p_{V}^{'}(v|\lambda)
\end{equation}
Plugging the back, we have 
\begin{equation}
-\frac{\alpha}{\sqrt{\delta}} \iint \frac{1}{\lambda^2} \mathcal{M}_{\psi,1}^{'}(v; \frac{\alpha}{u\sqrt{\delta}\lambda^2})p^{'}_{V}(v|\lambda)p_{\Lambda}(\lambda)\text{d}v \text{d}g \text{d}\lambda = -\frac{\alpha}{\sqrt{\delta}} \mathbb{E}[\frac{1}{\lambda}\mathcal{M}_{\psi,1}^{'}(v; \frac{\alpha}{u\sqrt{\delta}\lambda^2}) \cdot \frac{p_{V}^{'}(v|\lambda)}{\lambda p_{V}(v|\lambda)}]
\end{equation}
Therefore, we have essentially proved the following identity
\begin{equation}
\label{eq: GIP identity}
  \frac{\sqrt{\delta}}{\alpha}\mathbb{E}[\frac{\alpha H}{\sqrt{\delta}\Lambda} \mathcal{M}_{\psi,1}^{'}(B+\frac{\alpha}{\sqrt{\delta}\Lambda}H; \frac{\alpha}{u\sqrt{\delta}\Lambda^2})]  =-\frac{\alpha}{\sqrt{\delta}} \mathbb{E}[\frac{1}{\lambda}\mathcal{M}_{\psi,1}^{'}(v; \frac{\alpha}{u\sqrt{\delta}\lambda^2}) \cdot \frac{p_{V}^{'}(v|\lambda)}{\lambda p_{V}(v|\lambda)}]
\end{equation}

Using Cauchy Schwartz inequality, we have that
\begin{equation}
\mathbb{E}[\frac{1}{\lambda}\mathcal{M}_{\psi,1}^{'}(v; \frac{\alpha}{u\sqrt{\delta}\lambda^2}) \cdot \frac{p_{V}^{'}(v|\lambda)}{\lambda p_{V}(v|\lambda)}]^2 \leq \mathbb{E}[(\frac{1}{\lambda}\mathcal{M}_{\psi,1}^{'}(v; \frac{\alpha}{u\sqrt{\delta}\lambda^2}))^2] \mathbb{E}[(\frac{p_{V}^{'}(v|\lambda)}{\lambda p_{V}(v|\lambda)})^2]
\end{equation}
If we let $\mathcal{I}_{\Lambda}(V_{\alpha}) := \mathbb{E}\left[ \left(\frac{\xi_{V_{\alpha}}({V_{\alpha}|\Lambda)}}{\Lambda}\right)^2\right]$. Using the optimality conditions, we can show that the following inequality holds 
\begin{equation}
\frac{u^2\delta(1-\delta)^2}{\alpha^2} \leq \left(u^2(1-\delta) - \frac{\delta\sigma^2u^2}{\alpha^2}\right)\mathcal{I}_{\Lambda}(V_{\alpha})   
\end{equation}
Note that the inequality is independent of $\Psi$ and is true for every $(\alpha_{\psi},u_{\psi})$ for which the system of equations is satisfied.
Simplifying the above inequality, we get
\begin{equation}
u^2\left( \alpha^2(1-\delta)\mathcal{I}_{\Lambda}(V_{\alpha}) - \delta\sigma^2\mathcal{I}_{\Lambda}(V_{\alpha}) - \delta(1-\delta)^2\right) \geq 0   
\end{equation}
One can verify that $u>0$ since it's a Lagrange multiplier of an active constraint. Eliminating $u$ gives the following inequality
\begin{equation}
\alpha^2(1-\delta)\mathcal{I}_{\Lambda}(V_{\alpha}) - \delta\sigma^2 \mathcal{I}_{\Lambda}(V_{\alpha}) - \delta(1-\delta)^2 \geq 0
\end{equation}
Writing the inequality in terms of $h(\alpha)$, we arrive at
\begin{equation}
    1\geq h(\alpha)
\end{equation}
As we have previously established, $h(\alpha)$ is monotonically decreasing and $\alpha_{*}$ is the unique solution of $h(\alpha_{*}) = 1$. Therefore $\alpha_{*}$ is the smallest $\alpha$ that satisfies the above inequality, and since the above inequality holds for every convex potential $\Psi$ whose optimality conditions are given by \eqref{eq: KKT apx}, we have that $\alpha_{\Psi} \geq \alpha_{*}$.

\end{proof}

\subsection{Proof of Corollary \ref{cor: closed lower bound}}
\begin{corollary}
Let Assumptions \ref{ass: HDA},\ref{ass: Gaussian feat and noise} and \ref{ass: true dist} hold and $\alpha_{*}$ be defined as the solution to \eqref{eq: lower bound}, if $\Lambda = 1$ almost surely then
\begin{equation}
 \label{eq: closed lower bound apx}
 \alpha_{*}^2 \geq \frac{\sigma^2}{1-\delta} + \frac{1-\delta}{\mathcal{I}(B)}
\end{equation}
whenever, the Fisher information $\mathcal{I}(B)$ is well defined. The inequality becomes an equality if and only if $B$ is a Gaussian.
\end{corollary}

\begin{proof}
From the previous theorem, we have that
\begin{equation}
  \alpha_{*}^2= \frac{\delta\sigma^2}{1-\delta}+\frac{\delta(1-\delta)}{\mathcal{I}_{\Lambda}(V_{\alpha_{*}})}   
\end{equation}
If $\Lambda = 1$ a.s and $\mathcal{I}(B)$ is well defined, then 
\begin{equation}
    \mathcal{I}_{\Lambda}(V_{\alpha_{*}}) = \mathcal{I}(B + \frac{\alpha_{*}}{\sqrt{\delta}}H) \leq \frac{\mathcal{I}(B)}{1+\frac{\alpha_{*}^2}{\delta}\mathcal{I}(B)} 
\end{equation}
where the inequality is obtained from Stam's inequality and is strict when $B$ is a Gaussian. Plugging back, we have that 
\begin{equation}
    \alpha_{*}^2 \geq \frac{\delta\sigma^2}{1-\delta}+\frac{(1-\delta)(\delta + \alpha_{*}^2\mathcal{I}(B))}{\mathcal{I}(B)} 
\end{equation}
Re-arranging the terms gives us the desired lower bound
\begin{equation}
    \alpha_{*}^2 \geq \frac{\sigma^2}{1-\delta} + \frac{1-\delta}{\mathcal{I}(B)}
\end{equation}
\end{proof}

\subsection{Proof of Theorem \ref{thm: optimal potential}}
\begin{theorem}(Optimal $\Psi$)
Let Assumptions \ref{ass: HDA}, \ref{ass: Gaussian feat and noise}, and \ref{ass: true dist} hold and $\alpha_{*}$ be defined as the solution to \eqref{eq: lower bound}. Consider the following function $\psi_{*} : \mathbb{R}^2 \xrightarrow[]{} \mathbb{R}$
\begin{equation}
 \label{eq: optimal potential apx}
 \psi_{*}(v,\lambda) :=  - \mathcal{M}_{\log(P_{V_{\alpha_{*}}}(v|\lambda))}\left(v;\frac{\alpha^2_{*}(1-\delta) - \delta\sigma^2}{\delta(1-\delta)\lambda^2}\right),
\end{equation}
if $P_{V_{\alpha_{*}}}(v|\lambda)$ is log-concave in $v$ and we define $\Psi_{*}(\bm \beta) = \sum_{i=1}^{n}\psi_{*}(\beta_i, \Sigma_{i,i})$, then
\begin{enumerate}
\item $\Psi_{*}(\bm \beta) \in \mathcal{C}_{\psi}$
\item $\alpha_*$ is a solution to the system of equations \eqref{eq: KKT} obtained using $\psi^*(v,\lambda)$ and is therefore the optimal convex implicit bias.
\end{enumerate}
\end{theorem}

\begin{proof}
By Proposition \ref{prop: ME FC}, we have that
\begin{multline}
    - \mathcal{M}_{\log(P_{V_{\alpha_{*}}}(v|\lambda))}\left(v;\frac{\alpha^2_{*}(1-\delta) - \delta\sigma^2}{\delta(1-\delta)\lambda^2}\right) =\\ \frac{\delta(1-\delta)\lambda^2}{\alpha^2_{*}(1-\delta) - \delta\sigma^2}\left( \left(\frac{v^2}{2} + \frac{\alpha^2_{*}(1-\delta) - \delta\sigma^2}{\delta(1-\delta)\lambda^2} \log P_{V_{\alpha_{*}}}(v|\lambda)\right)^* - \frac{v^2}{2}\right)
\end{multline}   
Showing $\psi_{*}(v,\lambda)$ is convex is equivalent to showing $\left( \left(\frac{v^2}{2} + \frac{\alpha^2_{*}(1-\delta) - \delta\sigma^2}{\delta(1-\delta)\lambda^2} \log P_{V_{\alpha_{*}}}(v|\lambda)\right)^* - \frac{v^2}{2}\right)$ is convex. First, we will verify that $\frac{v^2}{2} + \frac{\alpha^2_{*}(1-\delta) - \delta\sigma^2}{\delta(1-\delta)\lambda^2} \log P_{V_{\alpha_{*}}}(v|\lambda)$ is convex. By definition
\begin{equation}
\log P_{V_{\alpha_{*}}}(v|\lambda) = - \frac{\delta \lambda^2 v^2}{2\alpha^2_{*}} + \log \int_{-\infty}^{\infty} \exp{\left(\delta \lambda^2(2vb -b^2)/2\alpha^2_{*}\right)} P_{B}(b)db + c
\end{equation}
for some constant $c$. One can verify the convexity of $\log \int_{-\infty}^{\infty} \exp{\left(\delta \lambda^2(2vb -b^2)/2\alpha^2_{*}\right)} P_{B}(b)db$ by double differentiation with respect to $v$. Therefore, it sufficient if $\frac{v^2}{2} - \frac{\alpha^2_{*}(1-\delta) - \delta\sigma^2}{\alpha^2_{*}(1-\delta)}\frac{v^2}{2}$ is convex, which is trivially true. Therefore, we have now verified the convexity of $\frac{v^2}{2} + \frac{\alpha^2_{*}(1-\delta) - \delta\sigma^2}{\delta(1-\delta)\lambda^2} \log P_{V_{\alpha_{*}}}(v|\lambda)$.

Next, we'll use the following property of the derivatives of convex conjugates from \cite{rockafellar-1970a} (Cor. 23.5.1), which says that if $f(x)$ is convex, then
\begin{equation}
    (f^*)'(x) = (f')^{-1}(x)
\end{equation}
Using the property of the derivative of an inverse, we further have that 
\begin{equation}
    (f^*)''(x) = \frac{1}{f''\left((f')^{-1}(x)\right)}
\end{equation}
Using the above property, taking double derivative of $\psi_{*}(v,\lambda)$ gives us
\begin{equation}
    \psi_{*}''(v,\lambda) = c_{*}\left(\frac{1}{1 + c_{*} (\log P_{V_{\alpha_{*}}}(v|\lambda))''(g(v))} - 1\right)
\end{equation}
where $g(v):= (v + c_{*}(\log P_{V_{\alpha_{*}}}(g(v)|\lambda))')^{-1}(v)$. Since $P_{V_{\alpha_{*}}}(v|\lambda)$ is log-concave by assumption, $(\log P_{V_{\alpha_{*}}}(v|\lambda))''(v) \leq 0$ for all $v$. This implies that $\psi_{*}''(v,\lambda) \geq 0$, finishing the proof on the sufficient condition.

Next, we verify optimality conditions \eqref{eq: KKT apx} to prove that the optimal convex potential is given by Theorem \ref{thm: optimal potential}. Let $\alpha_{*}$ be the solution to \eqref{eq: lower bound apx}. Now, consider the following candidate for optimal potential
\begin{equation}
\label{eq: candidate}
    \mathcal{M}^{'}_{\psi_{*},1}(v;\frac{\alpha^2_{*}(1-\delta) - \delta\sigma^2}{\delta(1-\delta)\lambda^2}) =  -\frac{P_{V_{\alpha_{*}}}^{'}(v|\lambda)}{P_{V_{\alpha_{*}}}(v|\lambda)}
\end{equation}
We now show that $\alpha_{*}$ and $u_{*}:= \frac{\alpha_{*}\sqrt{\delta}(1-\delta)}{\alpha^2_{*}(1-\delta) - \delta\sigma^2}$ satisfy, the optimality conditions for the above candidate. Plugging in \eqref{eq: KKT apx}(b), we get
\begin{align}
    \mathbb{E}\left[\left(\frac{1}{\Lambda} \mathcal{M}_{\psi_{*},1}^{'}(B+\frac{\alpha_{*}}{\sqrt{\delta}\Lambda}H; \frac{\alpha_{*}}{u_{*}\sqrt{\delta}\Lambda^2})\right)^2\right] &= \mathcal{I}_{\Lambda}(V_{\alpha_{*}})\\
    &= u_{*}^2\frac{(\alpha^2_{*}(1-\delta) - \delta\sigma^2)^2}{\alpha_{*}^2\delta(1-\delta)^2}\mathcal{I}_{\Lambda}(V_{\alpha_{*}})\\
    &=u_{*}^2(1-\delta) - \frac{\delta \sigma^2 u_{*}^2}{\alpha_{*}^2}
\end{align}
where the second equality is from the definition of $u_{*}$ and the third equality is from \eqref{eq: lower bound apx}. Next, we verify \eqref{eq: KKT apx}(a)
\begin{align}
    \mathbb{E}\left[\frac{H}{\Lambda} \mathcal{M}_{\psi,1}^{'}(B+\frac{\alpha_{*}}{\sqrt{\delta}\Lambda}H; \frac{\alpha_{*}}{u_{*}\sqrt{\delta}\Lambda^2})\right] &= -\frac{\alpha_{*}}{\sqrt{\delta}} \mathbb{E}[\frac{1}{\lambda}\mathcal{M}_{\psi,1}^{'}(v; \frac{\alpha_{*}}{u_{*}\sqrt{\delta}\lambda^2}) \cdot \frac{p_{V}^{'}(v|\lambda)}{\lambda p_{V}(v|\lambda)}] \\
    &= \frac{\alpha_{*}}{\sqrt{\delta}}\mathcal{I}_{\Lambda}(V_{\alpha_{*}})\\
    &=u_{*}(1-\delta)
\end{align}
where the first equality is due to the identity \eqref{eq: GIP identity} and the rest follows from definitions of $u_{*}$ and $\mathcal{I}_{\Lambda}(V_{\alpha_{*}})$. Now that we have shown that \eqref{eq: candidate} satisfies optimality conditions, taking anti derivative of \eqref{eq: candidate} gives $\mathcal{M}_{\psi_{*}}(v;\frac{\alpha^2_{*}(1-\delta) - \delta\sigma^2}{\delta(1-\delta)\lambda^2}) = -\log P_{V_{\alpha_{*}}}(v|\lambda)$ which can be inverted to give us the optimal convex potential whenever the $P_{V_{\alpha_{*}}}$ is log-concave by Proposition \ref{prop: ME inv}.

\end{proof}

\subsection{Proof Corollary \ref{cor: optimal for Gaussian}}

\begin{corollary}($\Psi_{*}$ for Gaussian $B$)
 Let Assumptions \ref{ass: HDA},\ref{ass: Gaussian feat and noise} and \ref{ass: true dist} hold and $B\sim \mathcal{N}(0,1)$, then the optimal implicit bias is given as 
 \begin{equation}
  \label{eq: optimal for Gaussian apx}
  \Psi_*(\bm \beta) = \bm \beta^{T}\bm\Sigma^{1/2}\left(\frac{\sigma^2}{1-\delta} \bm I_{n} + \bm\Sigma\right)^{-1}\bm\Sigma^{1/2} \bm \beta.
 \end{equation}
\end{corollary}

\begin{proof}
When $B\sim \mathcal{N}(0,1)$, then $P_{V_{\alpha_{*}}}(v|\lambda)$ is Gaussian density
\begin{equation}
    P_{V_{\alpha_{*}}}(v|\lambda) = c_{0}\exp{\left(-\frac{v^2}{2(1+\frac{\alpha_{*}^2}{\delta\lambda^2})}\right)}
\end{equation}
where $c_0$ is a constant independent of $v$. Note that $P_{V_{\alpha_{*}}}(v|\lambda)$ is log-concave, by Theorem \ref{thm: optimal potential}, we have that the optimal potential is given as 

\begin{align}
    \psi_{*}(v,\lambda) &=  - \mathcal{M}_{\log(P_{V_{\alpha_{*}}}(v|\lambda))}\left(v;\frac{\alpha^2_{*}(1-\delta) - \delta\sigma^2}{\delta(1-\delta)\lambda^2}\right)
\end{align}
Solving the Moreau envelope gives us
\begin{equation}
    \psi_{*}(v,\lambda) = \frac{\lambda^2v^2}{2(\lambda^2 + \frac{\delta\sigma^2}{1-\delta})} + c_1
\end{equation}
Here, $c_1$ is independent of $v$, so we can ignore it. In the vectorized form,  $\Psi_{*}(\bm \beta) = \sum_{i=1}^{n}\psi_{*}(\beta_i, \Sigma_{i,i})$ gives the desired result.
 \begin{equation}
  \Psi_*(\bm \beta) = \frac{1}{2}\bm \beta^{T}\bm\Sigma^{1/2}\left(\frac{\sigma^2}{1-\delta} \bm I_{n} + \bm\Sigma\right)^{-1}\bm\Sigma^{1/2} \bm \beta.
 \end{equation}
\end{proof}

\end{document}